\documentclass{article}

\usepackage[numbers]{natbib}

\usepackage[preprint]{neurips_2025}

\usepackage[utf8]{inputenc} % allow utf-8 input
\usepackage[T1]{fontenc}    % use 8-bit T1 fonts
\usepackage{hyperref}       % hyperlinks
\usepackage{url}            % simple URL typesetting
\usepackage{booktabs}       % professional-quality tables
\usepackage{amsfonts}       % blackboard math symbols
\usepackage{nicefrac}       % compact symbols for 1/2, etc.
\usepackage{microtype}      % microtypography
\usepackage{xcolor}         % colors

\usepackage{makecell}

\usepackage{titletoc}
\newcommand\DoToC{%
  \startcontents
  \printcontents{}{1}{\hrulefill\vskip0pt}
  \vskip0pt \noindent\hrulefill
  }

\usepackage{graphicx}
\usepackage{subfigure}

\usepackage{amsmath}
\usepackage{amssymb}
\usepackage{mathtools}
\usepackage{amsthm}

\theoremstyle{plain}
\newtheorem{theorem}{Theorem}[section]
\newtheorem{proposition}[theorem]{Proposition}
\newtheorem{lemma}[theorem]{Lemma}

\theoremstyle{definition}
\newtheorem{definition}[theorem]{Definition}
\newtheorem{assumption}[theorem]{Assumption}
\newtheorem{remark}[theorem]{Remark}
\newtheorem{example}[theorem]{Example}

\usepackage{wrapfig}  % make figure lying in the right of text

\usepackage{algorithm}
\usepackage{algorithmic}
\usepackage{float}  
\usepackage{subfigure}  
\usepackage{multirow} 
\usepackage{makecell}
\usepackage{xcolor}
\usepackage{enumitem}
\usepackage{fdsymbol}
\usepackage{dsfont}
\usepackage{newfloat}
\usepackage{listings}
\usepackage{enumitem}

\def\Vbar{{\perp\!\!\!\perp}}
\def\NotVbar{\not{\perp\!\!\!\perp}}

\def\bR{{\mathbb R}}
\def\bE{{\mathbb E}}

\def\x{\mathbf{x}}
\def\u{\mathbf{u}}
\def\s{\mathbf{s}}

\def\X{\mathbf{X}}
\def\U{\mathbf{U}}
\def\S{\mathbf{S}}

\def\cX{{\mathcal X}}
\def\cU{{\mathcal U}}
\def\cS{{\mathcal S}}

\def\cR{{\mathcal R}}

\def\cF{{\mathcal F}}

\title{Long-term Causal Inference via Modeling Sequential Latent Confounding}

\author{%
  Weilin Chen \\
  School of Computer Science\\
  Guangdong University of Technology\\
  Guangzhou, China\\
  \texttt{chenweilin.chn@gmail.com}
  \And
  Ruichu Cai\thanks{Corresponding author. Ruichu Cai is also with the Peng Cheng Laboratory, Shenzhen, China.} \\
  School of Computer Science\\
  Guangdong University of Technology\\
  Guangzhou, China\\
  \texttt{cairuichu@gmail.com}
  \AND
  Yuguang Yan \\
  School of Computer Science\\
  Guangdong University of Technology\\
  Guangzhou, China\\
 \texttt{ygyan@gdut.edu.cn}
  \And
  Zhifeng Hao \\
  College of Science \\
  Shantou University\\
  Shantou, China\\
   \texttt{haozhifeng@stu.edu.cn}
  \AND
  José Miguel Hernández-Lobato\\
  Department of Engineering\\
  University of Cambridge\\
  Cambridge CB2 1PZ, United Kingdom\\
  \texttt{jmh233@cam.ac.uk}
}

\begin{document}

\maketitle

\begin{abstract}
Long-term causal inference is an important but challenging problem across various domains.
To solve the latent confounding problem in long-term observational studies, existing methods leverage short-term experimental data.
\citet{ghassami2022combining} propose an approach based on the Conditional Additive Equi-Confounding Bias (CAECB) assumption, 
which asserts that the confounding bias in the short-term outcome is equal to that in the long-term outcome, so that the long-term confounding bias and the causal effects can be identified. 
While effective in certain cases, this assumption is limited to scenarios where there is only one short-term outcome with the same scale as the long-term outcome.
%with a one-dimensional short-term outcome.
In this paper, we introduce a novel assumption that extends the CAECB assumption to accommodate temporal short-term outcomes.
Our proposed assumption states a functional relationship between sequential confounding biases across temporal short-term outcomes,
under which we theoretically establish the identification of long-term causal effects.
Based on the identification result, we develop an estimator and conduct a theoretical analysis of its asymptotic properties.
Extensive experiments validate our theoretical results and demonstrate the effectiveness of the proposed method.
\end{abstract}

\section{Introduction}
\label{intro}

Long-term causal inference is an important but challenging problem across various scientific fields, such as education \cite{athey2019surrogate}, medicine \cite{fleming1994surrogate}, and marketing  \cite{hohnhold2015focusing}.
While in many real-world scenarios, long-term observational data are readily available, a major challenge for long-term causal inference is the presence of latent confounding in observational studies.
A common way to mitigate this issue is to incorporate short-term experimental data,
which raises a fundamental question:
how can short-term experimental data be leveraged to address latent confounding in observational data for long-term causal inference, as shown in Fig. \ref{fig: causal graph}(a) and \ref{fig: causal graph}(b)?

Existing works explore various methods to mitigate latent confounding by utilizing observational and experimental data based on different assumptions \footnote{Please kindly refer to Appendix \ref{app: discuss assums} for detailed discussion on more existing assumptions.}. 
One widely used assumption is the Latent Unconfoundedness assumption \cite{athey2020combining, chen2023semiparametric, meza2021nested}, 
which posits the short-term potential outcomes mediate the long-term potential outcome in the observational data, i.e., $Y(a) \Vbar A \mid  \S(a), \X, G=O$.
These methods are effective when short-term outcomes $\S$ contain substantial information about the latent confounder $\U$, 
such as in studies on the lifetime effects of youth employment and training programs in the United States \cite{aizer2024lifetime}.
However, this assumption essentially restricts that the latent confounders can not affect the long-term outcome, i.e., ruling out the causal edge $\U \rightarrow Y$.
To address this limitation, a follow-up work \cite{ghassami2022combining} proposes a novel assumption called Conditional Additive Equi-Confounding Bias (CAECB), allowing for the causal link $\U \rightarrow Y$.
The CAECB assumption states that the confounding bias in the short-term outcome is equal to that in the long-term outcome, enabling the identification of long-term causal effects.
This assumption may be more reasonable and more aligned with practical settings where confounding biases share a similar manner over short- and long-term outcomes conditional on covariates $\X$,
such as in research studying the impact of school finance reforms on student outcomes \cite{jackson2016effects}.

Following the work of \citet{ghassami2022combining}, we focus on the setting where the causal link $\U \rightarrow Y$ exists.
Beyond this, we further consider temporal short-term outcomes, as shown in Figure \ref{fig: causal graph}(c), a more common scenario where the method under the CAECB assumption \cite{ghassami2022combining} is not applicable.
In many real-world applications, 
short-term outcomes exhibit temporal dependencies, and capturing these relationships is essential for inferring long-term causal effects.
For example, in evaluating the long-term (e.g. year-long) effectiveness of the medication, patients undergo regular follow-up visits (e.g., weekly or monthly). 
During these visits, temporal short-term health indicators are recorded as short-term outcomes.
These sequential measurements capture the progression of patients’ conditions over time, providing valuable information on the long-term outcome of interest.
Consequently, the absence of temporal considerations in existing methods limits their effectiveness and constrains the potential of long-term causal inference in practice.

\begin{figure}[!t] 
	\centering
    % \vspace{-.4cm}  
    \includegraphics[width=1.\textwidth]{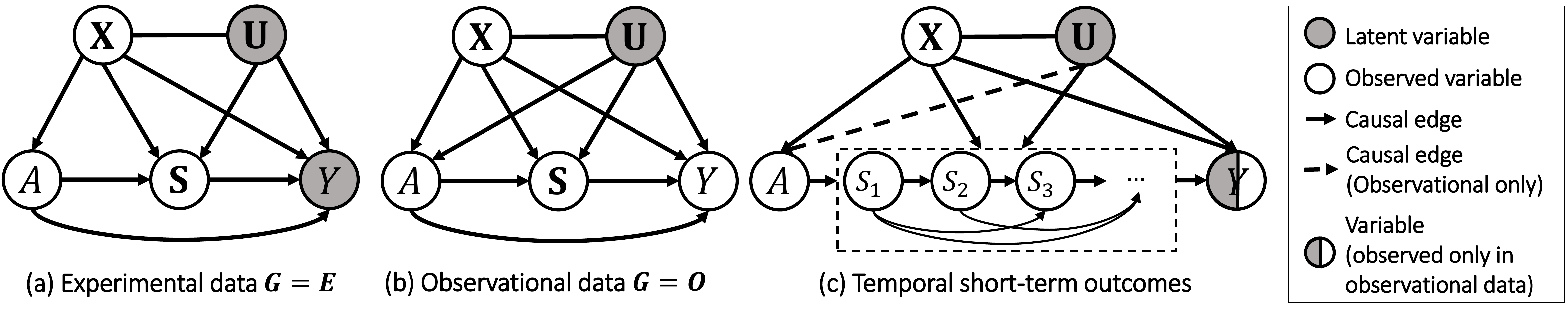}
    % \vspace{-.3cm}
	\caption{
    Causal graphs for experimental and observational data with $\X$ as covariates, $\U$ as latent confounders, $A$ as treatment, $\S$ as short-term outcomes, and $Y$ as the long-term outcome.
    Fig. \ref{fig: causal graph}(a) shows the experimental data, where $A$ is unaffected by $\U$ and $Y$ cannot be unobserved.
    Fig. \ref{fig: causal graph}(b) depicts the observational data, where $\U$ affects $A$, $\S$, and $Y$, and $Y$ can be observed.
    Fig. \ref{fig: causal graph}(c) illustrates the full causal graph with temporally extended short-term outcomes.
    }
    \label{fig: causal graph} 
    % \vspace{-.3cm}
\end{figure}

In this paper, we introduce a novel assumption called \textbf{Functional CAECB (FCAECB)}, which extends the existing CAECB assumption to capture the temporal dependencies among temporal short-term outcomes.
Roughly speaking, the proposed FCAECB assumption posits a functional relationship between sequential latent confounding biases.
Under this assumption, we establish the theoretical identification of long-term causal effects.
Correspondingly, we devise an algorithm for the long-term effect estimation. 
Additionally, we analyze the convergence rates of our proposed estimator in the asymptotic and finite sample setting within a generic nonparametric regression framework, 
with the ultimate goal of deepening the understanding of how sequential short-term confounding biases contribute to inferring long-term effects.
Overall, our contribution can be summarized as follows:
\begin{itemize}[leftmargin=8pt]
    \item \textbf{Novel Assumption for Long-Term Causal Inference:} We study the problem of long-term causal inference in the presence of temporal short-term outcomes.
    We propose a novel assumption named FCAECB for identifying long-term causal effects,
    which enables capturing the time-dependent relationships between sequential latent confounding biases. Note that the existing CAECB assumption can be seen as our special case.
    \item \textbf{Estimator for Heterogeneous Long-Term Effects:} We devise an estimator for estimating heterogeneous long-term effects under the proposed FCAECB assumption, 
    which can be implemented using any machine learning regression method. 
    Theoretically, we analyze the asymptotic properties of our estimator.
    \item \textbf{Empirical Validation:} We conduct extensive experiments to validate our theoretical findings and demonstrate the effectiveness of the proposed estimator.
\end{itemize}

\section{Related Work}
\label{related work}

\textbf{Long-term causal inference}
For years, researchers have investigated which short-term outcomes can reliably predict long-term causal effects \cite{prentice1989surrogate, frangakis2002principal, lauritzen2004discussion, gilbert2008evaluating, chen2007criteria,ju2010criteria,yin2020novel}.
% Various criteria have been proposed for identifying valid surrogates.
%including the Prentice criteria \cite{prentice1989surrogate}, principal surrogacy \cite{frangakis2002principal}, strong surrogate criteria \cite{lauritzen2004discussion}, causal effect predictiveness \cite{gilbert2008evaluating}, and consistent surrogate and its variations \cite{chen2007criteria, ju2010criteria, yin2020novel}.
Recently, there has been growing interest in estimating long-term causal effects using surrogates, which is also the focus of this paper.
\underline{One line of work} assumes the unconfoundedness assumption.
Under the unconfoundedness assumption, LTEE \cite{cheng2021long} and Laser \cite{cai2024long} are based on specifically designed neural networks for long-term causal inference.
EETE \cite{kallus2020role} explores the data efficiency from the surrogate in several settings and proposes an efficient estimator for treatment effects.
ORL \cite{tran2023inferring} introduces a doubly robust estimator for average treatment effects using only short-term experiments, additionally assuming stationarity conditions between short-term and long-term outcomes.
\citet{singh2024double} proposes a kernel ridge regression-based estimator for long-term effect under continuous treatment.
Additionally, \citet{yang2024learning, wu2024policy} develop a policy learning method for balancing short-term and long-term rewards.
Our work is \textbf{different} from theirs. We do not rely on the unconfoundedness assumption, and we use the data combination technique to solve the problem of unobserved confounders.
\underline{Another line of work}, which avoids the unconfoundedness assumption, tackles the latent confounding problem by combining experimental and observational data.
This setting is first introduced by \citet{athey2019surrogate}, 
which, under surrogacy assumption, constructs the so-called Surrogate Index (SInd) as the substitutions for long-term outcomes in the experimental data for effect identification.
As follow-up work, \citet{athey2020combining} introduce the latent unconfoundedness assumption, which assumes that short-term potential outcomes can mediate the long-term potential outcomes,  thereby enabling long-term causal effect identification.
Under this assumption, \citet{meza2021nested, yang2024estimating, chen2023semiparametric} propose several estimators for effect estimation.
The alternative feasible assumptions \cite{ghassami2022combining} are proposed to replace the latent unconfoundedness assumption, e.g., the Conditional Additive Equi-Confounding Bias (CAECB) assumption.
Based on proximal methods \cite{tchetgen2024introduction}, \citet{imbens2022long} propose considering the short-term outcomes as proxies of latent confounders, thereby achieving effect identification.
%\textbf{However}, these studies primarily concentrate on average treatment effects, whereas our focus in this paper is on heterogeneous effects.
Among the existing literature, the most closely related work is the work of \citet{ghassami2022combining}.
Our work can be viewed as a \textbf{significant extension} of theirs, as we consider a more practical scenario where short-term outcomes exhibit temporal dependencies, and theoretically, their CAECB assumption is a special case of our proposed assumption.
Our work is also closely related to the work of \citet{imbens2022long}, which also explores the sequential short-term outcomes and assumes them as proxies of latent confounders. Different from that, we explore the sequential information from short-term confounding bias, which allows for inferring long-term confounding bias and results in the final effect identification. 

\textbf{Modeling Latent Confounding Bias}
An effective way to solve the latent confounding problem in the data combination setting is to model latent confounding bias.
\citet{kallus2018removing} proposes modeling confounding bias under a linearity assumption.
\citet{hatt2022combining} introduce to model the nonlinear confounding bias using the representation learning technique.
\citet{wu2022integrative} propose the integrative R-learner via a regularization for the conditional effects and confounding bias with the Neyman orthogonality.
\citet{zhou2025two} propose a two-stage representation learning strategy to model such a confounding bias.
\textbf{Different} from these works, we focus on the long-term causal inference setting, and rather than focus on how to model the confounding bias, we concentrate more on the relationship between sequential confounding biases.

\section{Preliminaries}
\label{Notations}

\subsection{Notations, Problem Definition, Basic Assumptions}

Let 
$A \in \{0,1\}$ be the treatment variable, 
$\X \in \cX \subseteq  \bR^{d}$ be the observed covariates where $d$ is the dimension of $\X$,
$\U \in \cU \subseteq \bR^{d_U}$ be the latent confounders where  $d_U$ is the dimension of $\U$.
Let $\S = [ S_1, S_2, \dots , S_T ] $ be the short-term outcome variable where $S_t \in \cS \subseteq \bR $ is the short-term outcome measured at time step $t$,
and $Y=S_{T+\mu} \in \cR \subseteq \bR $ be the long-term outcome.
Further, we leverage the potential outcome framework \cite{rubin1978bayesian}.
We denote
$\S(a)$ as the potential short-term outcome, 
$S_t(a)$ as the potential short-term outcome at time step $t$,
and $Y(a)$ as the potential long-term outcome.
Following existing work on long-term inference \cite{athey2020combining, ghassami2022combining, imbens2022long, chen2023semiparametric}, 
we denote $G\in \{E,O\}$ as the indicator of data group, where $G=E$ indicates experimental data and $G=O$ indicates observational data. 
Let lowercase letters (e.g., $a, \x, \u, \s, y, s(a), y(a)$) denote the value of the previously described random variables.
Let the superscript $(i)$ denote a specific unit, e.g., $\x^{(i)}$ is the covariate value of unit $i$. 
Then, the experimental data and the observational data are denoted as $\mathbb D_e =\{a^{(i)}, \x^{(i)}, \s^{(i)}, g^{(i)}=E\}_{i=1}^{n_e}$ and $\mathbb D_o = \{a^{(i)}, \x^{(i)}, \s^{(i)}, y^{(i)}, G^{(i)}=O\}_{i=n_e+1}^{n_e+n_o}$, where $n_e$ and $n_o$ are the size of experimental data and the observational data respectively. 
%Let the whole dataset denote as $\mathbb D = \mathbb D_e \cup \mathbb D_o$ of the size $n$ satisfying $n=n_e+n_o$.

For ease of convenience, we denote the following nuisance functions and confounding bias:
\begin{equation} \label{defined nuisance} 
    \begin{aligned}
      &  \mu_{S_t}^E(A,\X) = \bE [\S_t\mid A,\X,G=E], \quad \mu_{S_t}^O(A,\X) = \bE [\S_t\mid A,\X,G=O], \\
      &  \mu_Y^E(A,\X) = \bE [Y\mid A,\X,G=E], \quad  \mu_Y^O(A,\X) = \bE [Y\mid A,\X,G=O], \\
      & \omega_t(\X) =  \mu_{S_t}^E(1,\X) - \mu_{S_t}^E(0,\X) + 
      \mu_{S_t}^O(0,\X) - \mu_{S_t}^O(1,\X),
    \end{aligned}
\end{equation}
where $\omega_t(\X)$ is known as the \textit{confounding bias},  the discrepancy between the conditional mean outcome differences derived from the experimental data and the observational data (Please see Sections \ref{sec: CAECB assumption} and \ref{sec: extension Identification} for more details on how it serves to the identification of long-term effects).

Moreover, we denote stochastic boundedness with $O_p$ and convergence in probability with $o_p$. We denote $X_1 \Vbar X_2$ as the independence between $X_1$ and $X_2$. We use $a_n \asymp b_n$ to denote both $a_n / b_n$ and $b_n / a_n$ are bounded. We use $a_n \lesssim b_n$ to denote both $a_n \leq C b_n$ for some constant $C > 0$.

\textbf{Task}: Given a short-term experimental dataset $\mathbb D_e $ and a long-term observational dataset $\mathbb D_o $, the goal in this paper is to identify and estimate the Heterogeneous Long-term Causal Effects (HLCE), i.e.,
\begin{equation}
    \tau (\x) = \bE [Y(1)-Y(0)\mid \X=\x]. 
\end{equation}
However, HLCE $\tau (\x)$ is not identifiable from the experimental data alone, since $Y$ is missing in that dataset.
Also it is not identifiable from the observational data alone, since the observational data suffers from the latent confounding problem, i.e., $A \NotVbar \{Y(a),\S(a)\}\mid \X, G=O$. 
Furthermore, the information regarding the causal effects in the experimental data is not necessarily relevant to that in the observational data without further assumptions.
To ensure the identification, we first make the following assumptions that are commonly assumed in long-term causal inference \cite{  athey2019surrogate, athey2020combining, chen2023semiparametric, ghassami2022combining, imbens2022long}: 

\begin{assumption}[Consistency] \label{assum: consist}
% If a unit is assigned treatment, we observe its associated potential outcome. 
% Formally, if $A=a$, then $Y=Y(a)$, and $\S=\S(a)$.
If $A=a$, then $Y=Y(a)$, and $\S=\S(a)$.
\end{assumption}

\begin{assumption}[Positivity] \label{assum: positi}
% The treatment assignment is non-deterministic. 
% Formally, $\forall a,\x$, we have $0<P(A=a\mid \X=\x)<1$, and $0<P(G=O\mid A=a, \X=\x)<1$.
$\forall a,\x$, $0<P(A=a\mid \X=\x);P(G=O\mid A=a, \X=\x)<1$.
\end{assumption}

\begin{assumption} [Observational Data]%[Weak internal validity of observational data] 
\label{assum: internal validity of obs}
% Latent confounders exist in observational data. 
% Formally, $\forall a$, we have $A \Vbar \{Y(a),\S(a)\}\mid \X, \U, G=O$ and  $A \NotVbar \{Y(a),\S(a)\}\mid \X, G=O$.
$\forall a$, $A \Vbar \{Y(a),\S(a)\}\mid \X, \U, G=O$. % and  $A \NotVbar \{Y(a),\S(a)\}\mid \X, G=O$.
\end{assumption}

\begin{assumption} [Experimental Data]%[Internal validity of experimental data]
\label{assum: internal validity of exp}
% There are no latent confounders in experimental data.
% Formally, $\forall a$, we have $A\Vbar \{Y(a),\S(a)\}\mid \X, G=E$.
$\forall a$, $A\Vbar \{Y(a),\S(a)\}\mid \X, G=E$.
\end{assumption}

\begin{assumption} [Data Combination] %[External validity of experimental data] 
\label{assum: external validity of exp}
% The distribution of the potential outcomes is invariant to whether the data belongs to experimental or observational data. Formally, $\forall a$, we have $G\Vbar \{Y(a),\S(a)\}\mid \X$.
 $\forall a$, $G\Vbar \{Y(a),\S(a)\}\mid \X$.
\end{assumption}

Assumptions \ref{assum: consist} and \ref{assum: positi} are standard assumptions in causal inference \cite{rubin1974estimating, imbens2000role}. 
Assumptions \ref{assum: internal validity of obs}, \ref{assum: internal validity of exp} and \ref{assum: external validity of exp} are mild and widely assumed in data combination settings \cite{shi2023data, imbens2022long, athey2019surrogate, athey2020combining, hu2023identification}. 
Specifically, Assumption \ref{assum: internal validity of obs} allows the existence of latent confounders in observational data, thus it is much weaker than the traditional unconfoundedness assumption.
Assumption \ref{assum: internal validity of exp} is reasonable and can hold since the treatment assignment mechanism is under control in the experiments.
Assumption \ref{assum: external validity of exp} connects the potential outcome distributions between observational and experimental data. 

However, Assumptions \ref{assum: consist}, \ref{assum: positi}, \ref{assum: internal validity of obs}, \ref{assum: internal validity of exp} and \ref{assum: external validity of exp} are still not sufficient to identify the causal estimand of interest. 
The root cause is that, even though the assumptions above link the experimental and observational data, the (long-term) latent confounding problem remains unsolved.
In the following section, we first review a method proposed by \citet{ghassami2022combining}, which poses an extra assumption, called the Conditional Additive Equi-Confounding Bias (CAECB) assumption, to achieve the identification of HLCE. 
Then, we propose our approach that generalizes the method of \citet{ghassami2022combining} to allow temporal short-term outcomes.

\begin{figure*}[!t]
    \centering
    % \vspace{-.4cm}
    \subfigure[Schematic representation of CAECB assumption \cite{ghassami2022combining}.]
    {\includegraphics[width=.44\textwidth]{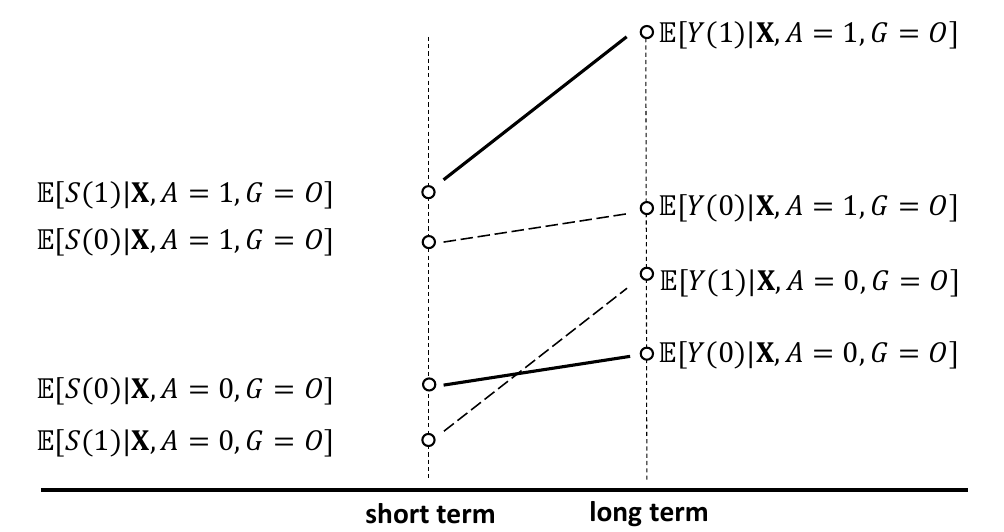}
     \label{fig: existing assumption}
    }
    \hspace{2pt}
    \subfigure[Schematic representation of our proposed FCAECB assumption.]{\includegraphics[width=.53\textwidth]{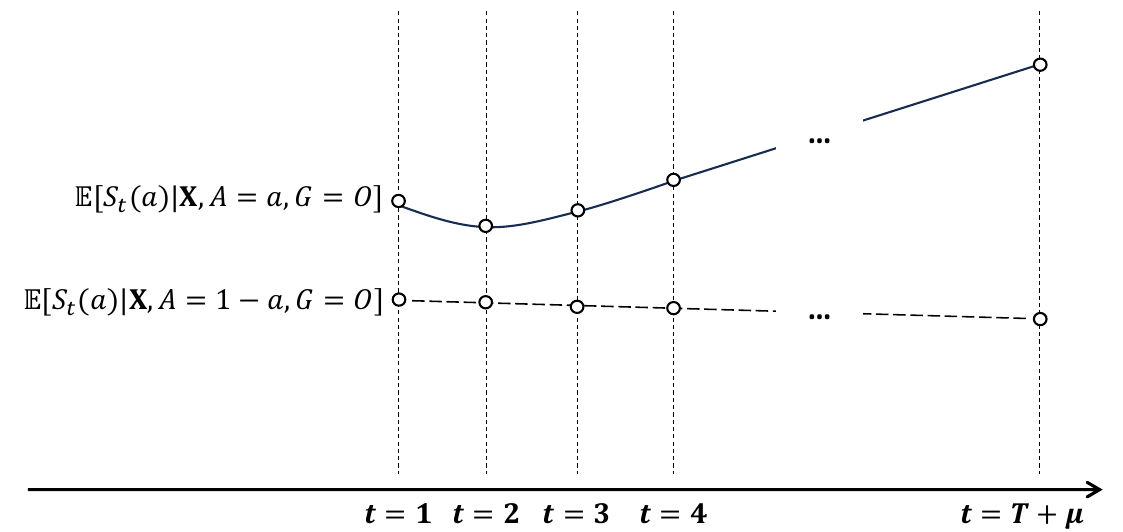}
        \label{fig: our assumption}
    }
    \caption{Schematic representations of CAECB assumption proposed by \citet{ghassami2022combining} and our FCAECB assumption. As shown in Figure \ref{fig: existing assumption}, the CAECB assumption requires that the confounding bias in the short-term outcome is equal to that in the long-term outcome. As shown in Figure \ref{fig: our assumption}, the FCAECB assumption relaxes this constraint by allowing for temporal short-term outcomes and only requiring that confounding biases across different time steps follow a specific pattern rather than remaining equal.}
    \label{fig: assumption}
    % \vspace{-.3cm}
\end{figure*}

\subsection{CAECB Assumption \cite{ghassami2022combining}}
\label{sec: CAECB assumption}

\citet{ghassami2022combining} introduced a method for combining experimental and observational data to identify the HLCE under an extra CAECB assumption.
The method under CAECB assumption is only applicable to scenarios with a one-dimensional short-term outcome, which we denote as $S$ (and its corresponding potential outcome $S(a)$) with slightly abusing notation.

To begin with, we first restate the CAECB assumption:

\begin{assumption} [Conditional Additive Equi-Confounding Bias, CAECB \cite{ghassami2022combining}] \label{assum: equ bias}
The difference in conditional expected values of short-term potential outcomes across treated and control groups is the same as that of the long-term potential outcome variable.
Formally, $\forall a$, we have 
\begin{equation} \label{eq: equ bias}
% \resizebox{0.94\linewidth}{!}{$
\begin{aligned}
    & \bE [S(a)\mid \X,A=0,G=O]-\bE[S(a)\mid \X,A=1,G=O]
    \\ = & \bE [Y(a)\mid \X,A=0,G=O]-\bE[Y(a)\mid \X,A=1,G=O].
\end{aligned}
% $}
\end{equation}

\end{assumption}

\begin{theorem} \label{theo: identifi} Suppose Assumptions \ref{assum: consist}, \ref{assum: positi}, \ref{assum: internal validity of obs}, \ref{assum: internal validity of exp}, \ref{assum: external validity of exp} and \ref{assum: equ bias} hold, then $\tau(\x)$ can be identified: % as follows:
    \begin{equation} \label{eq: identi} 
      \begin{aligned}
       \tau(\x) 
            = & \mathbb E[Y(1) - Y(0)\mid \X=\x]
            = \mu_Y^O(1,\x) - \mu_Y^O(0,\x)  + \omega(\x),
      \end{aligned}
    \end{equation}
    where $\omega(\x) = \mu_S^E(1,\x) -  \mu_S^E(0,\x)
             +  \mu_S^O(0,\x) - \mu_S^O(1,\x)$
             is the short-term confounding bias.
\end{theorem}

Proof can be found in Appendix \ref{app: proof of iden}. 
A similar identification result in terms of long-term average causal effects has been established by \citet{ghassami2022combining}. For completeness, we present the identification result of the HLCE $\tau(\x)$ in Theorem \ref{theo: identifi}.

\begin{remark} [Interpretation on Assumption \ref{assum: equ bias}]
Assumption \ref{assum: equ bias} is intuitively illustrated in Figure \ref{fig: existing assumption}.
Assumption \ref{assum: equ bias} can be seen as a generalization of the parallel assumption in the difference-in-differences (DiD) framework \cite{angrist2009mostly,ashenfelter1984using},
Note that Assumption \ref{assum: equ bias} essentially implies that the short-term confounding bias is the same as the long-term one.
In this way, $\omega(\x)$ can be interpreted as the long-term confounding bias, resulting in the identification result in Eq. \eqref{eq: identi}.
\end{remark}

Considering Eq. \eqref{eq: equ bias} in Assumption \ref{assum: equ bias}, $S$ and $Y$ should be on the same scale, which restricts the practical application of the method under this assumption.
In the next section, we extend this assumption to allow for temporal short-term outcomes, which significantly improve the practical utility.

\section{Long-term Identification under Functional CAECB Assumption}
\label{sec: extension Identification}

In this section, we introduce a novel assumption considering the temporal information between temporal short-term outcomes, enabling the identification of long-term causal effects. 

Specifically, we formalize our proposed assumption:

\begin{assumption} [Functional Conditional Additive Equi-Confounding Bias, FCAECB] 
\label{assum: time series equ bias}
The difference in conditional expected values of short-term potential outcomes across treated and control groups between two time steps follows the learnable function form $f: \bR \rightarrow \bR$.
Formally, 
$\forall a$, we have 
\begin{equation} \label{eq: time series equ bias}
\begin{aligned}
    b_{t+1}(a,\X) = f(\X) b_t(a,\X) ,
    % & \bE [S(a)\mid \X,A=0,G=O]-\bE[S(a)\mid \X,A=1,G=O]
    % \\ = & \bE [Y(a)\mid \X,A=0,G=O]-\bE[Y(a)\mid X,A=1,G=O].
\end{aligned}
\end{equation}
where $b_t(a,\X)=\bE [S_t(a)\mid \X,A=0,G=O]-\bE[S_t(a)\mid \X,A=1,G=O]$.
\end{assumption}

\begin{remark} [Interpretation on Assumption \ref{assum: time series equ bias}]
We illustrate our proposed FCAECB assumption in Figure \ref{fig: our assumption}. 
We relax the CAECB assumption \cite{ghassami2022combining} to allow for the temporal short-term outcomes, instead of restricting that the short-term outcome should be the same scale as the long-term outcome.
Additionally, the existing CAECB assumption can be seen as our special case when $f(\X)$ in the FCAECB assumption satisfies $f(\X)=1$.
\end{remark}

To provide a better understanding of our Assumption \ref{assum: time series equ bias}, we provide the insight in term of the functional form $f(\X)$ in the following proposition.

\begin{proposition} \label{propo: confounding bias function}
    Under Assumption \ref{assum: time series equ bias},  $\forall t$ the confounding biases between times $t$ and $t+1$ follow
\begin{equation}\label{eq: propo}
\begin{aligned}
    \omega_{t+1}(\X) = f(\X) \omega_t(\X) ,
\end{aligned}
\end{equation}
where $\omega_t(\X)$ is the confounding bias at time step $t$, defined as $\omega_t(\X) =  \mu_{S_t}^E(1,\X) - \mu_{S_t}^E(0,\X) + 
      \mu_{S_t}^O(0,\X) - \mu_{S_t}^O(1,\X)$.
\end{proposition}

The proof is given in Appendix \ref{app: proof propo}.
\begin{remark} [Insight of Assumption \ref{theo: time series identifi}]
As stated in Proposition \ref{propo: confounding bias function}, Assumption \ref{assum: time series equ bias} essentially states the confounding biases $\omega_t(\X)$ between adjacent time steps follow the functional form $f(\X)$.
Proposition \ref{propo: confounding bias function} also illustrates the way of how to learn the function $f(\X)$ using the observed variables, unlike the definition in Eq. \eqref{eq: time series equ bias} defining $f(\X)$ using the potential outcomes (See Section \ref{Estimation} for the estimation based on Proposition \ref{propo: confounding bias function}).
\end{remark}

\begin{remark} [Testing Assumption \ref{theo: time series identifi} over Short Term Duration]
It is possible to empirically test Assumption \ref{assum: time series equ bias} over the short-term duration, which provides insights into whether it can be extrapolated to long-term confounding bias. 
This is particularly reasonable when short- and long-term outcomes share the same practical meaning, making it more likely that the confounding bias mechanism remains stable.
Specifically, suppose the short-term outcomes are observed at $3$ time steps. 
Then, given $S_1,S_2,S_3$, we can estimate $f(\mathbf{x})$ using $S_1$ and $S_2$, and evaluate the $R^2$ measure on $S_2$ and $S_3$ where $R^2 = 1- \frac{\sum_i^N (\omega_3(\mathbf{x}^{(i)}) -\hat{f}(\mathbf{x}^{(i)})\omega_2(\mathbf{x}^{(i)}) )^2}{\sum_i^N(\omega_3(\mathbf{x}^{(i)}) - \frac{1}{N}\sum_j^N\omega_3(\mathbf{x}^{(j)}) )^2}$. 
If $R^2\approx 1$, it suggests that $f$ well explains the relationships of confounding biases between adjacent time steps. 
Conversely, if $R^2 \ll 1$, then the assumption may not hold, or $f$ is poorly fitted.
\end{remark}

\begin{example}
In many real-world scenarios, Assumption \ref{assum: time series equ bias} is reasonable.
In the motivating example from our introduction, AIDs patients undergo regular follow-up visits, and their short-term health indicators, such as viral loads or CD4 counts, are recorded as short-term outcomes. Since CD4 counts are critical indicators of health status, they serve as both short- and long-term outcomes. In such cases, Assumption \ref{assum: time series equ bias} can be tested over short-term durations and is likely to hold.
\end{example}

The key to identifying long-term causal effect under Assumption \ref{assum: time series equ bias} is using the temporal information in $f(\X)$ to extrapolate the long-term confounding bias $\omega_{T+\mu}(\X)$.
Using Eq. \eqref{eq: propo}, the long-term confounding bias $\omega_{T+\mu}(\X)$ can be expressed as $\omega_{T+\mu}(\X)=f^\mu(\X)\omega_{T}(\X)$, which results in the following theorem of long-term effect identification.

\begin{theorem} \label{theo: time series identifi} 
Suppose Assumptions \ref{assum: consist}, \ref{assum: positi}, \ref{assum: internal validity of obs}, \ref{assum: internal validity of exp}, \ref{assum: external validity of exp} and \ref{assum: time series equ bias} hold, then the heterogeneous long-term effects $\tau(\x)$ can be identified as follows:
    \begin{equation} \label{eq: time series identi}  
      \begin{aligned}
     \tau(\x) 
            = & \mathbb E[Y(1)-Y(0)\mid \X=\x] 
            =  \mu_Y^O(1,\x) - \mu_Y^O(0,\x) +  f^\mu(\x) \omega_T(\x),
      \end{aligned}
    \end{equation}
    where $\omega_T(\x)=\mu_{S_T}^E(1,\x) -  \mu_{S_T}^E(0,\x)
             +  \mu_{S_T}^O(0,\x) - \mu_{S_T}^O(1,\x) $ is the short-term confounding bias at time step $T$.
\end{theorem}

The proof is given in Appendix \ref{app: proof of time series iden}. Theorem \ref{theo: time series identifi} provides the identification result of the heterogeneous long-term effects $\tau (\X)$.
The identification result consists of two parts: 1. long-term outcome differences in observational data $\mu_{S_T}^O(0,\x) - \mu_{S_T}^O(1,\x)$, and 2. long-term confounding bias $f^\mu(\x) \omega_T(\x)$. The long-term confounding bias is identified by the extrapolated result using short-term confounding bias $ \omega_ T(\x)$ under Assumption \ref{assum: time series equ bias}. 
The identification result also illustrates a way to estimate the long-term effects $\tau(\x)$ via modeling sequential latent confounding, which inspires our estimator as shown in Section \ref{Estimation}.

\section{Long-term Causal Effect Estimation}
\label{Estimation}

In this section, we first introduce our estimator $\hat \tau (\X)$ for heterogeneous long-term effects under our proposed Assumption \ref{assum: time series equ bias}, and provide the corresponding theoretical analysis of the proposed estimator.

\subsection{Estimator} \label{sec: alg}

Our estimator $\hat \tau (\X)$ directly follows the identification result in Theorem \ref{theo: time series identifi}.
As shown in Eq. \eqref{eq: time series identi}, the estimators consist of three nuisance components, outcome mean difference between treated and control group in observational data $\mu_Y^O(1,\x) - \mu_Y^O(0,\x) $, confounding bias $\omega_T(\x)$, and the function between confounding bias $f(\x)$.
The first and the second terms can be directly estimated by fitting nuisance functions, and the third term $f(\x)$ should be estimated based on the fitted confounding biases.
Accordingly, we design our heterogeneous long-term effect estimator within a two-stage regression framework, which are model-agnostic algorithms that decompose the estimation task into multiple sub-problems, each solvable using any supervised learning/regression methods.
 
Specifically, our proposed estimator follows:

\begin{itemize}[leftmargin=23px]
    \item [S0.] (Optional) Selecting subsets of short-term outcomes $\S$, yielding appropriate $T$ and $\mu$; 
    \item [S1.] Fitting the following nuisance functions: 
    $\hat\mu_Y^O(a,\x)$, $\hat\mu_{S_t}^E(a,\x)$, and $\hat\mu_{S_t}^O(a,\x)$ for all $a$ and $t$;
    \item [S2.] Constructing the confounding bias  $\hat \omega_t(\x)=\hat \mu_{S_t}^E(1,\x) - \hat  \mu_{S_t}^E(0,\x)
             +  \hat \mu_{S_t}^O(0,\x) - \hat \mu_{S_t}^O(1,\x) $
    for all $t$, and fitting the function $\hat f(\x)$ by minimizing 
    \begin{equation} \label{eq: s2}
        arg\!\min _{f \in \cF } \Sigma_{t=1}^{T-1} \Sigma_{i=1}^n\left( \hat \omega_{t+1}(\x_i) -  f (\x_i) \hat \omega_t(\x_i) \right)^2;
    \end{equation}
    \item [S3.] Constructing the final HLCE estimator as
    \begin{equation}
      \begin{aligned}
        & \hat \tau(\x) 
        =  \hat \mu_Y^O(1,\x) - \hat\mu_Y^O(0,\x) +  \hat f^\mu(\x)  \left( \hat \mu_{S_T}^E(1,\x) -  \hat \mu_{S_T}^E(0,\x)
             +  \hat \mu_{S_T}^O(0,\x) - \hat \mu_{S_T}^O(1,\x) \right) .
    \end{aligned}
    \end{equation}
\end{itemize}

Note that, in addition to the fitting/constructing steps S1-S3, we also introduce an optional step S0 to select appropriate subsets of short-term outcomes $\S$.
This step is motivated by the identification result in Eq. \eqref{eq: time series identi} that allows for different choices of $\mu$ and $T$.
For example, suppose we can observe $6$-step short-term outcome $\S=[S_1,S_2,\dots,S_6]$ and the long-term outcome of interest is $Y=S_9$. Then we have multiple choices of $T$ and $\mu$, e.g.,
using all short-term outcomes $\S=[S_1,S_2,\dots,S_6]$ with $T=6$ and $\mu=3$,
or using $\S=[S_1, S_3, S_5]$ with $T=3$ and $\mu=2$, 
or using $\S=[S_1, S_5]$ with $T=2$ and $\mu=1$.
We will discuss how to choose $T$ and $\mu$ in the next section.

\subsection{Convergence Rate Analyses}
\label{Theoretical Analyses}

In this paper, we assume the smoothness assumption of the estimated functions, where $s$-smooth functions are contained in the H\"older ball $\mathcal H _d(s)$, estimable with the minimax rate \cite{stone1980optimal} of $n^{\frac{1}{2+d/s}}$ where $d$ is the dimension of $\mathcal X$. Formally, we provide the following definition.

\begin{definition}[H\"older ball]  \label{def: holder ball}
The H\"older ball $\mathcal H _d(s)$ is the set of $s$-smooth functions $f: \mathbb R^d \rightarrow \mathbb R$ supported on $\mathcal X \subseteq \mathbb R^d$ that are $\lfloor s \rfloor $-times continuously differentiable with their multivariate partial derivatives up to order $\lfloor s \rfloor $ bounded, and for which
\begin{equation*}
    | \frac{\partial ^m f }{\partial^{m_1} \cdots \partial^{m_d} }(x) - \frac{\partial ^m f }{\partial^{m_1} \cdots \partial^{m_d} }(x^\prime) | \lesssim \| x - x^\prime\|^{s- \lfloor s \rfloor }_2,
\end{equation*}
$\forall x, x^\prime$ and $m=(m1, \cdots, m_d)$ such that $\Sigma_{j=1}^d m_j = \lfloor s \rfloor$. 
\end{definition}

Accordingly, we assume the smoothness assumption of the nuisance functions, as well as the boundedness assumption of the nuisance functions and their estimates as follows.

\begin{assumption}[Smoothness Assumption] \label{asmp: smooth}
    We assume that the nuisance functions and the mapping $f$ 
    defined in Assumption \ref{assum: time series equ bias} 
    satisfy: (1) $\mu_{S_t}^E$, $\mu_{S_t}^O$, and $\mu_Y^O$ are $\alpha$-smooth, $\beta$-smooth, and $\gamma$-smooth, respectively, and all are estimable at \citet{stone1980optimal}'s minimax rate of $n^{\frac{-p}{2p+d}}$ for a $p$-smooth function; (2) $f(x)$ is $\eta$-smooth.
\end{assumption}

\begin{assumption}[Boundness Assumption] \label{asmp: bound}
    We assume that the nuisance functions $\mu_{S_t}^O(a,x),\mu_{S_t}^E(a,x)$ and their estimates are bounded, i.e., $\forall t$, $|\mu_{S_t}^O(a,x)| < C_1$, $|\hat \mu_{S_t}^O(a,x)| < C_2$, $|\mu_{S_t}^E(a,x)| < C_3$ and $|\hat \mu_{S_t}^E(a,x)| < C_4$ hold for some $C_1, C_2, C_3, C_4 >0$.   
\end{assumption}

% Assumptions \ref{asmp: smooth} and \ref{asmp: bound} are common and standard for the heterogeneous effect estimation analysis, e.g., \cite{kennedy2023towards, curth2021nonparametric}. Specifically, Assumption \ref{asmp: smooth} quantifies how difficult to non-parametrically fit the nuisance functions. Assumption \ref{asmp: bound} is standard in analyzing the error bound of the estimator.

We now state our main theoretical results.
To obtain our error bounds of the proposed estimator, we leverage the same sample splitting technique from \cite{kennedy2023towards}, which randomly splits the datasets into two independent sets and applies them to the regressions of the first step and second step respectively. Such a technique is originally used to analyze the convergence rate of the double robust conditional average treatment effect estimation in the traditional setting \cite{kennedy2023towards} and later is adapted to several other methods \cite{curth2021nonparametric, frauenestimating}. Different from them, we use such a technique for the sequential latent confounding modeling, which is then adapted for the long-term effect estimation.

To begin with, we first provide the rate of $\hat f(\x)$ as follows:

\begin{lemma} \label{lemma: convergence of f} 
    Suppose the training steps S1 and S2 are train on two independent datasets of size $n$ respectively, and suppose Assumptions \ref{assum: consist}, \ref{assum: positi}, \ref{assum: internal validity of obs}, \ref{assum: internal validity of exp}, \ref{assum: external validity of exp}, \ref{assum: time series equ bias}, \ref{asmp: smooth}, and \ref{asmp: bound} hold, then we have
        \begin{equation} \label{eq: f rate}   
      \begin{aligned}
       \hat f(\x) - & f(\x) 
         =  O_p \left( 
            (\frac{1}{(T-1)n})^{\frac{\eta}{2\eta+d}} 
            % \right.\\  & \left. 
            +  (\frac{1}{(T-1)n})^{\frac{\alpha}{2\alpha+d}}
            +  (\frac{1}{(T-1)n})^{\frac{\beta}{2\beta+d}}
            \right),
      \end{aligned}
    \end{equation}
    which attains the oracle rate if 
     $\min \{\alpha, \beta \} \geq \eta$.
\end{lemma}

The proof is given in Appendix \ref{app: proof f conver}.
Since $\hat f (\x)$ is a time-series model, its effective sample is $(T-1)n$, thus it might achieve a faster rate if we observed a longer duration of short-term outcomes.
Moreover, the rate of $\hat f(\x)$ consist two part: the oracle rate $(\frac{1}{(T-1)n})^{\frac{\eta}{2\eta+d}} $, and the rate of fitting nuisance functions $(\frac{1}{(T-1)n})^{\frac{\alpha}{2\alpha+d}} + (\frac{1}{(T-1)n})^{\frac{\beta}{2\beta+d}}$. If the nuisance functions $\mu_{S_t}^E$ and $\mu_{S_t}^O$ is smooth enough such that $\min \{\alpha, \beta \} \geq \eta$, then $\hat f (\x)$ will attain the oracle rate.

Based on Lemma \ref{lemma: convergence of f}, we provide the rate of $\hat \tau(\x)$ in the following theorem.

\begin{theorem} \label{theo: convergence} 
Suppose Lemma \ref{lemma: convergence of f} hold, then we have
    \begin{equation} \label{eq: tau rate}   
      \begin{aligned}
       & \hat \tau(\x) -   \tau(\x)  \\ 
          = &  O_p \left(  
          n^{-\frac{\gamma}{2\gamma+d}} 
          + n^{-\frac{\alpha}{2\alpha+d}} 
          + n^{-\frac{\beta}{2\beta+d}}
           + \frac{\mu}{ ((T-1)n)^{\frac{\eta}{2\eta+d}} }  
           % \right. \\ & \left.
           + \frac{\mu}{ ((T-1)n)^{\frac{\alpha}{2\alpha+d}} }  + 
           \frac{\mu}{ ((T-1)n)^{\frac{\beta}{2\beta+d}} }  
            \right).
      \end{aligned}
    \end{equation}
\end{theorem}

Proof can be found in Appendix \ref{app: proof tau conver}. 
Theorem \ref{theo: convergence} follows directly from Lemma \ref{lemma: convergence of f}.
The rate of $\hat \tau (\x)$ consist of three components: (a) $n^{-\frac{\gamma}{2\gamma+d}} $ represents the rate of $\mu_Y^O(a,\x)$; (b) $n^{-\frac{\alpha}{2\alpha+d}} + n^{-\frac{\beta}{2\beta+d}}$ corresponds the rate of confounding bias $\omega_T(\x)$; (c) and the remaining term is the rate of $\hat f^\mu (\x)$. 
Theorem \ref{theo: convergence} indicates that the overall convergence rate is largely governed by the smoothness parameters $\alpha, \beta, \gamma, \eta$ and the temporal information $T$ and $\mu$, where $T$ denotes the maximum observed duration of short-term outcomes $\S$, and $\mu$ is the time horizon of the long-term outcome of interest.
This result aligns with intuition: both $T$ and $\mu$ affect the convergence behavior. More accurate effect estimation benefits from longer short-term outcome durations $T$, while requiring the long-term horizon $\mu$ not to be too distant.

\begin{remark}[Choosing $\mu$ and $T$]
    Note that, the bound in Eq. \eqref{eq: tau rate} contains the term $\frac{\mu}{(T-1)^\frac{\alpha}{2\alpha+d}}$ (similar for $\beta$ and $\eta$).
    If we have prior knowledge of the smoothness of nuisance functions, i.e., knowing $\alpha$ ($\beta$ and $\eta$), then we can optimally select $\mu$ and $T$ to obtain a faster rate.
    For instance, if all estimated functions are sufficiently smooth, i.e., $\alpha,\beta,\eta \rightarrow \infty$, then $\mu$ and $T$ should be chosen to minimize $\frac{\mu}{\sqrt{T-1}}$, since $\frac{\alpha}{2\alpha+d}, \frac{\beta}{2\beta+d}, \frac{\eta}{2\eta+d}  \rightarrow 1/2$.
    As a concrete example, suppose we can observe $6$-step short-term outcome $\S=[S_1,S_2,\dots,S_6]$ and the long-term outcome of interest is $Y=S_9$. The optimal choice is to use $\S=[S_1,S_5]$ only to estimate $\tau(\x)$, since it result in a minimum $\frac{\mu}{\sqrt{T-1}}=1$ where $\mu=1$ and $T=2$.
    In practice, when no prior knowledge is available, it may be advisable to choose $\mu$ as small as possible, since a larger $\mu$-power in $\hat f ^\mu (\x)$ could lead to significant errors.
\end{remark}

\section{Experiments}
\label{Experiments}

In this section, we perform a series of experiments to evaluate the effectiveness and validity of our proposed estimator.
We first provide the experimental setup in Section \ref{sec: setup}, and provide the experimental results and corresponding analyses in Section \ref{sec: result}.

% Specifically, we answer the following research questions (RQs):
% \begin{itemize}
%     \item \textbf{RQ1 (Effectiveness):} Can the proposed estimator achieve accurate long-term effect estimation?
%     \item \textbf{RQ2 (Comparision Performance):} Can the proposed estimator outperform other methods in terms of long-term effect estimation?
%     \item \textbf{RQ3 (Sample Sensitivity):} Are our proposed methods sensitive to sample size?
% \end{itemize}

\subsection{Experimental Setup}
\label{sec: setup}

In this section, we describe the data generation process for the synthetic datasets, the baselines, the implementation details, and the metrics used in the experiments.

\textbf{Synthetic Datasets}
We partially follow the data generation setup in \cite{kallus2018removing} to induce a specific form of confounding bias and define nuisance functions. Details are provided in Appendix \ref{app: add exp}.
The confounding bias at time $t$ is designed as $\omega_t(\x) = -2^t \x$, aligning with our assumptions.
% We run several control experiments based on this setup.
By default, we set the observational sample size $n_o = 4000$, experimental sample size $n_e = 2000$, and $T = 6$, $\mu = 3$.
In variations, we change $\mu \in {1,2,3,4,5}$, $T \in {4,5,6,7,8}$, and fix the ratio $n_o:n_e = 2:1$ while varying $n_e \in \{1000,2000,3000,4000,5000,10000\}$ (implying $n_o \in \{2000,4000,6000,8000,10000,2000\}$).

% The data generation process is partly following \cite{kallus2018removing} such that we can obtain a specific form of confounding bias, as well as nuisance functions. 
% Please refer to Appendix \ref{app: add exp} for the detailed generation process.
% By design, the confounding bias at time step $t$ is $\omega_t(\x)=-2^t \x$, satisfying our assumption.
% We perform several control experiments using this data generation process.
% The default values are: the observational data sample size $n_o=4000$, experimental data sample size $n_e=2000$, $T=6$, and $\mu=3$.
% In control experiments, we vary $\mu$ within $\{1,2,3,4,5\}$ and $T$ within $T \in \{4,5,6,7,8\}$.
% We also fix the ratio of $n_O:n_E=2:1$, and varying $n_e$ within $\{1000,2000,3000,4000,5000,10000\}$ (which corresponds to $n_o$ values of $\{2000,4000,6000,8000,10000,20000\}$).

\textbf{Baselines and Implementation}
We compare our method against several baselines for long-term causal effect estimation, including LTEE \cite{cheng2021long}, Latent Unconfoundedness \cite{athey2020combining}, Sequential Outcomes \cite{imbens2000role}, NNPIV \cite{meza2021nested}, and CAECB \cite{ghassami2022combining}. For the Sequential Outcomes method, which requires three groups of short-term outcomes ($S_1, S_2, S_3$ in their paper), we test three grouping strategies: (1) $S_{\lceil T/2 \rceil -1}, S_{\lceil T/2 \rceil}, S_{\lceil T/2 \rceil +1}$, (2) $S_1, S_{\lceil T/2 \rceil}, S_T$, and (3) $S_{1:\lceil T/2 \rceil-1}, S_{\lceil T/2 \rceil}, S_{\lceil T/2 \rceil+1:T}$.
NNPIV is applied under both the latent unconfoundedness assumption (NNPIV\_LU) \cite{hatt2022combining} and the surrogacy assumption (NNPIV\_SInd) \cite{athey2019surrogate}.
For CAECB, which uses a single short-term outcome, we implement four variants using the first, middle, last, and a randomly selected outcome.
We also include the T-learner \cite{kunzel2019metalearners} trained on observational $Y$ and an idealized version trained on experimental $Y$ (infeasible in practice, as $Y$ is unobserved in experimental data), serving as an idealized benchmark denoted as $\hat \tau_{exp}$.
All nuisance functions are estimated using correctly specified regression. Our estimator is denoted as $\hat{\tau}$.
Additional details are provided in Appendix \ref{app: add exp}.

\textbf{Metrics} As for heterogeneous effect estimation, we report Precision in the Estimation of Heterogeneous Effect (PEHE) $\varepsilon_{PEHE}= \sqrt{\frac{1}{n} \Sigma_{i=1}^{n} (\tau(\x_i) - \hat \tau(\x_i))^2}$. As for average long-term causal effect estimation, we report the absolute error $\varepsilon_{ATE}= | \Sigma_{i=1}^{n} \tau(\x_i) -\Sigma_{i=1}^{n}  \hat \tau(\x_i)) |$. 
For all metrics, we report the mean values and standard deviation (in tables) or standard errors (in figures) by $50$ running.

\begin{table}[t]
\vspace{-.5cm}
\caption{The results on synthetic data are presented as the mean and standard deviation. Lower values indicate better performance. The best-performing methods are bolded, with the exception of the idealized estimator $\hat{\tau}_{exp}$, which is excluded from the comparison. (See Appendix \ref{app: add exp} for Full results.)}
\label{sample-table}
\vspace{-.3cm}
\begin{center}
\resizebox{\textwidth}{!}{ % 自动调整到页面宽度
\begin{tabular}{l|cc|cc|cc}
\toprule
Data 
& \multicolumn{2}{|c|}{$\mu=1$, $T=6$} 
& \multicolumn{2}{|c|}{$\mu=3$, $T=6$} 
& \multicolumn{2}{|c}{$\mu=5$, $T=6$} 
\\
\midrule
Metrics
& \multicolumn{1}{|c}{${\varepsilon_{PEHE}}$} 
& \multicolumn{1}{c|}{$\varepsilon_{ATE}$} 
& \multicolumn{1}{|c}{${\varepsilon_{PEHE}}$} 
& \multicolumn{1}{c|}{$\varepsilon_{ATE}$} 
& \multicolumn{1}{|c}{${\varepsilon_{PEHE}}$} 
& \multicolumn{1}{c}{$\varepsilon_{ATE}$} 
\\
\midrule
$\hat{\tau}_{exp}$ (idealized)
& 3.0033 $\pm$ 1.8226 & 2.0722 $\pm$ 1.7681
% & 6.0076 $\pm$ 3.6446 & 4.1457 $\pm$ 3.5353
& 12.0142 $\pm$ 7.2895 & 8.2901 $\pm$ 7.0709
% & 24.0297 $\pm$ 14.5790 & 16.5806 $\pm$ 14.1423
& 48.0590 $\pm$ 29.1580 & 33.1620 $\pm$ 28.2845  \\
\midrule
T-learner (using obs. $Y$)
& 71.8431 $\pm$ 3.1612 & 6.8211 $\pm$ 2.8357
% & 143.6844 $\pm$ 6.3217 & 13.6400 $\pm$ 5.6710
& 287.3693 $\pm$ 12.6434 & 27.2795 $\pm$ 11.3424
% & 574.7384 $\pm$ 25.2877 & 54.5614 $\pm$ 22.6837
& 1149.4769 $\pm$ 50.5746 & 109.1205 $\pm$ 45.3685   \\
LTEE
& 73.0421 $\pm$ 7.1531 & 10.0977 $\pm$ 6.1071
% & 142.1183 $\pm$ 10.9446 & 17.3357 $\pm$ 10.0487
& 287.3264 $\pm$ 28.5202 & 32.4417 $\pm$ 19.5739
% & 589.0638 $\pm$ 56.6554 & 79.3176 $\pm$ 41.2797
& 1168.7612 $\pm$ 122.7687 & 137.0583 $\pm$ 107.5567   \\
% Surrogate Index
% & - & 2.9972 $\pm$ 2.3146
% % & - & 5.9945 $\pm$ 4.6286
% & - & 11.9888 $\pm$ 9.2576
% % & - & 23.9772 $\pm$ 18.5151
% & - & 47.9542 $\pm$ 37.0302   \\
Latent Unconfoundedness
& - & 2.6322 $\pm$ 2.0120
% & - & 5.2648 $\pm$ 4.0238
& - & 10.5293 $\pm$ 8.0471
% & - & 21.0585 $\pm$ 16.0946
& - & 42.1179 $\pm$ 32.1874   \\
Sequential Outcomes ($S_{\lceil T/2 \rceil -1}, S_{\lceil T/2 \rceil }, S_{\lceil T/2 \rceil +1}$)
& - & 67.5881 $\pm$ 3.0107
% & - & 135.1792 $\pm$ 6.0220
& - & 270.3607 $\pm$ 12.0428
% & - & 540.7164 $\pm$ 24.0884
& - & 1081.4361 $\pm$ 48.1759    \\
Sequential Outcomes ($S_1, S_{\lceil T/2 \rceil }, S_T$)
& - & 67.5790 $\pm$ 3.0138
% & - & 135.1608 $\pm$ 6.0282
& - & 270.3237 $\pm$ 12.0553
% & - & 540.6429 $\pm$ 24.1133
& - & 1081.2884 $\pm$ 48.2252
\\
Sequential Outcomes ($S_{1:\lceil T/2 \rceil -1}, S_{\lceil T/2 \rceil }, S_{\lceil T/2 \rceil +1:T}$)
& - & 67.5797 $\pm$ 3.0151
% & - & 135.1622 $\pm$ 6.0309
& - & 270.3266 $\pm$ 12.0607
% & - & 540.6486 $\pm$ 24.1242
& - & 1081.3001 $\pm$ 48.2472
\\
NNPIV\_LU
& 70.7124 $\pm$ 2.5320 & 6.1302 $\pm$ 2.8888
% & 141.4309 $\pm$ 5.0655 & 12.2646 $\pm$ 5.7776
& 282.8595 $\pm$ 10.1320 & 24.5293 $\pm$ 11.5558
% & 565.7158 $\pm$ 20.2586 & 49.0542 $\pm$ 23.1130
& 1131.4331 $\pm$ 40.5209 & 98.1118 $\pm$ 46.2263   \\
NNPIV\_SInd
& 70.7155 $\pm$ 2.5329 & 6.1306 $\pm$ 2.8884
% & 141.4320 $\pm$ 5.0665 & 12.2626 $\pm$ 5.7782
& 282.8655 $\pm$ 10.1314 & 24.5240 $\pm$ 11.5558
% & 565.7307 $\pm$ 20.2627 & 49.0492 $\pm$ 23.1115
& 1131.4602 $\pm$ 40.5252 & 98.0989 $\pm$ 46.2232   \\
CAECB (using first $S$)
& 70.7257 $\pm$ 3.1153 & 6.7160 $\pm$ 2.7984
% & 142.5669 $\pm$ 6.2757 & 13.5348 $\pm$ 5.6335
& 286.2518 $\pm$ 12.5973 & 27.1744 $\pm$ 11.3048
% & 573.6209 $\pm$ 25.2417 & 54.4563 $\pm$ 22.6461
& 1148.3595 $\pm$ 50.5286 & 109.0154 $\pm$ 45.3308   \\
CAECB (using middle $S$)
& 67.3754 $\pm$ 2.9827 & 6.3982 $\pm$ 2.6884
% & 139.2165 $\pm$ 6.1417 & 13.2171 $\pm$ 5.5210
& 282.9014 $\pm$ 12.4626 & 26.8566 $\pm$ 11.1910
% & 570.2704 $\pm$ 25.1065 & 54.1385 $\pm$ 22.5317
& 1145.0089 $\pm$ 50.3933 & 108.6976 $\pm$ 45.2162  \\
CAECB (using last $S$)
& 36.1202 $\pm$ 2.0211 & 3.5778 $\pm$ 1.8603
% & 107.9443 $\pm$ 4.9841 & 10.2576 $\pm$ 4.6541
& 251.6244 $\pm$ 11.2383 & 23.8972 $\pm$ 10.2153
% & 538.9915 $\pm$ 23.8540 & 51.1791 $\pm$ 21.5097
& 1113.7291 $\pm$ 49.1287 & 105.7381 $\pm$ 44.1729  \\
CAECB (using random $S$)
& 59.8772 $\pm$ 11.9929 & 5.7194 $\pm$ 2.6871
% & 131.7174 $\pm$ 12.7309 & 12.5383 $\pm$ 5.3841
& 275.4020 $\pm$ 16.0775 & 26.1778 $\pm$ 10.9897
% & 562.7700 $\pm$ 26.2888 & 53.4597 $\pm$ 22.2995
& 1137.5094 $\pm$ 50.1542 & 108.0188 $\pm$ 44.9689  \\
Ours $\hat \tau$
&\textbf{ 3.1110 $\pm$ 1.8230 } & \textbf{ 2.1470 $\pm$ 1.7826} 
% & \textbf{ 6.5594 $\pm$ 3.5589 } & \textbf{ 4.4610 $\pm$ 3.5302 }
&\textbf{ 13.9945 $\pm$ 7.1682 } & \textbf{ 9.296 $\pm$ 7.0820} 
% & \textbf{ 29.9866 $\pm$ 15.0884 } & \textbf{ 19.3927 $\pm$ 14.4300 }
&\textbf{ 64.3076 $\pm$ 32.9868 } & \textbf{ 40.5425 $\pm$ 29.7264 }  \\
\bottomrule
\end{tabular}
}
\end{center}
\label{table:vary mu}
% \vskip -0.1in
\vspace{-.3cm}
\end{table}

\subsection{Experimental Results}
\label{sec: result}

\textbf{Comparison with Baselines}
We compare our method with existing methods using different choices of $T$ and $\mu$, and results are shown in Fig. \ref{table:vary mu} 
% (please see Appendix \ref{app: add exp} for the full results). 
As expected, our method consistently outperforms all baselines, primarily because it is built upon more appropriate assumptions.
In contrast, the baseline methods suffer from substantial biases resulting from incorrect assumptions.
Furthermore, we observe that our estimator closely matches the idealized estimator in both magnitude and trend, highlighting the effectiveness of the proposed method.

\textbf{Optimal Choice of $T$ and $\mu$}
We conduct experiments using different choices of $T$ and $\mu$. 
Specifically, to predict long-term effects, we consider the following subsets of $\S$: 
\begin{wrapfigure}{r}{0.5\textwidth} 
% \begin{figure}[!h]
    \vspace{-.3cm}
    \centering
    \subfigure[Heterogeneous effect estimation error.]
    {\includegraphics[width=.23\textwidth]{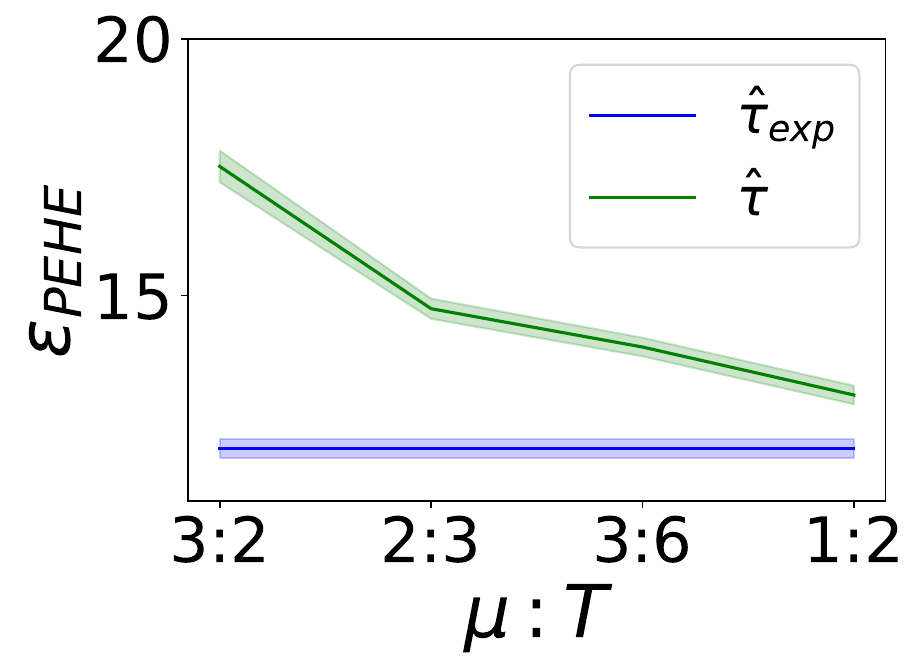}
     \label{fig: PEHE_OPTIMAL}
    }
    \hspace{1px}
    \subfigure[Average effect estimation error.]
    {\includegraphics[width=.23\textwidth]{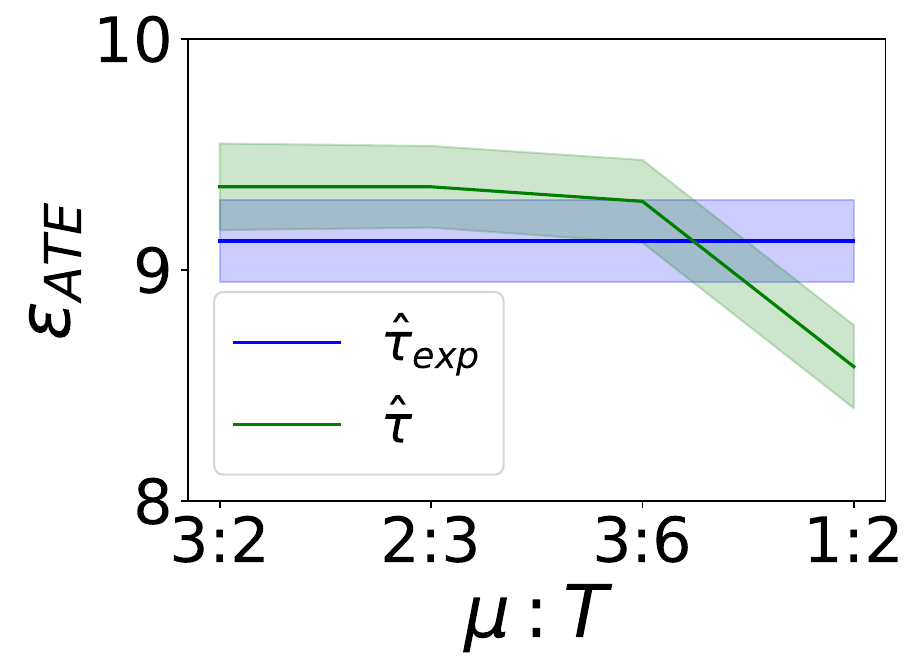}
        \label{fig: MAE_OPTIMAL}
    }
    \vspace{-.2cm}
    \caption{Results of the experiments in terms of different choice of $\mu$ and $T$.}
    \label{fig: optimal}
    \vspace{-.4cm}
% \end{figure}
\end{wrapfigure}
1. $\S=[S_1,S_3]$ with $T=2, \mu=3$; 
2. $\S=[S_1,S_3,S_5]$ with $T=3, \mu=2$; 
3. $\S=[S_1,S_2,S_3,S_4,S_5,S_6]$ with $T=6, \mu=3$; 
and 4. $\S=[S_1,S_5]$ with $T=2, \mu=1$.
The results, shown in Figure \ref{fig: optimal}, are sorted by the values of $\frac{\mu}{\sqrt{T-1}}$, which $3, 1.41, 1.34$, and $1$, respectively.
As shown in Figure \ref{fig: optimal}, as  $\frac{\mu}{\sqrt{T-1}}$ decrease, both the estimation errors $\varepsilon_{PEHE}$ and $\varepsilon_{ATE}$ decrease. This aligns with our theoretical findings in Theorem \ref{theo: convergence}, as a smaller $\frac{\mu}{\sqrt{T-1}}$ leads to a faster rate.
Interestingly, as shown in Figure \ref{fig: MAE_OPTIMAL}, we found that when the $\frac{\mu}{\sqrt{T-1}}$ is equal to $1$, our estimator $\hat \tau_(\x)$ achieves lower $\varepsilon_{ATE}$ than the idealized estimator $\hat \tau$. This may be attributed to the fact that our estimator $\hat \tau(\x)$ leverages observational data, leading to higher data efficiency.

\begin{figure}[!h]
    \vspace{-.2cm}
    \centering
    \subfigure[Heterogeneous effect estimation error with fixed $T$ and varying $\mu$.]
    {\includegraphics[width=.23\textwidth]{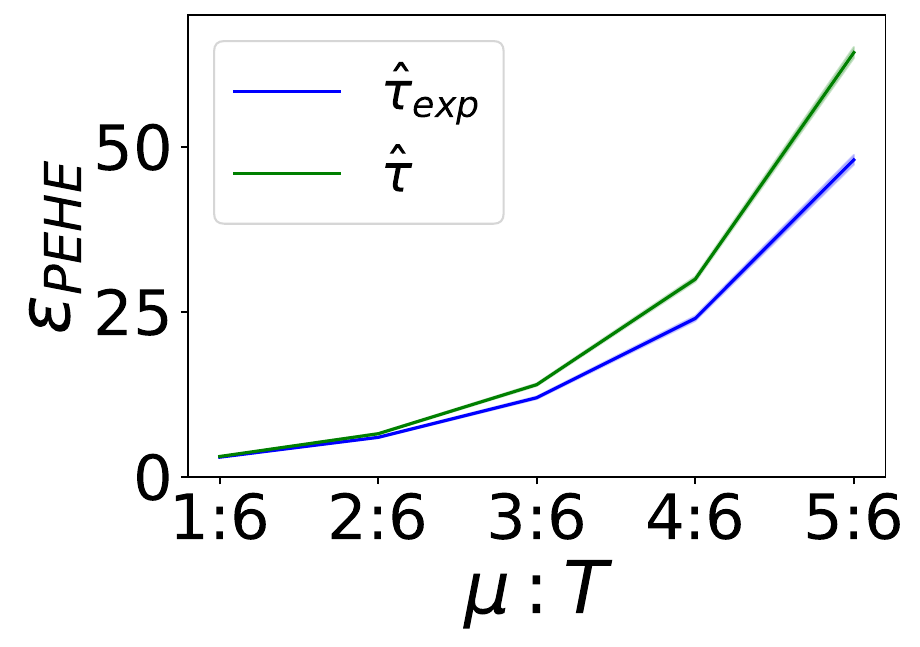}
     \label{fig: PEHE varyT}
    }
    \hspace{2pt}
    \subfigure[Average effect estimation error  with fixed $T$ and varying $\mu$.]
    {\includegraphics[width=.23\textwidth]{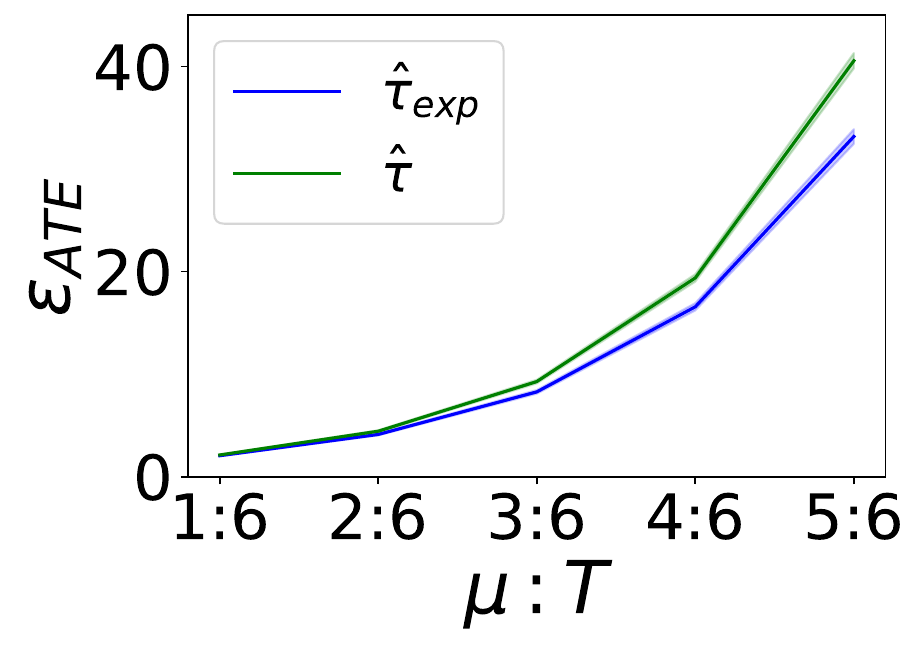}
        \label{fig: MAE varyT}
    }
    \hspace{3pt}
    \subfigure[Heterogeneous effect estimation error with fixed $\mu$ and varying $T$.]
    {\includegraphics[width=.23\textwidth]{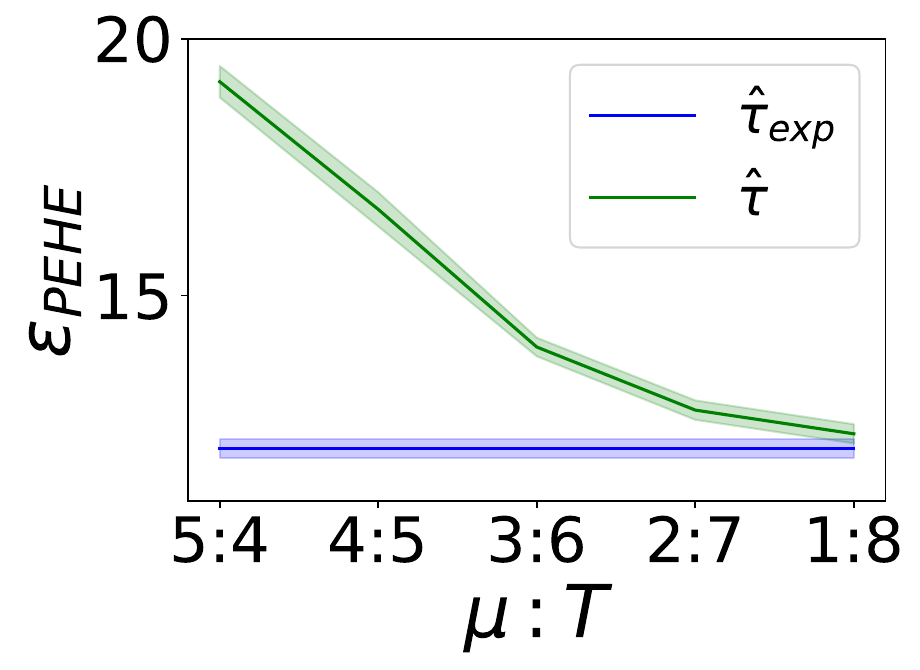}
     \label{fig: PEHE varymu}
    }
    \hspace{2pt}
    \subfigure[Average effect estimation error with fixed $\mu$ and varying $T$.]
    {\includegraphics[width=.23\textwidth]{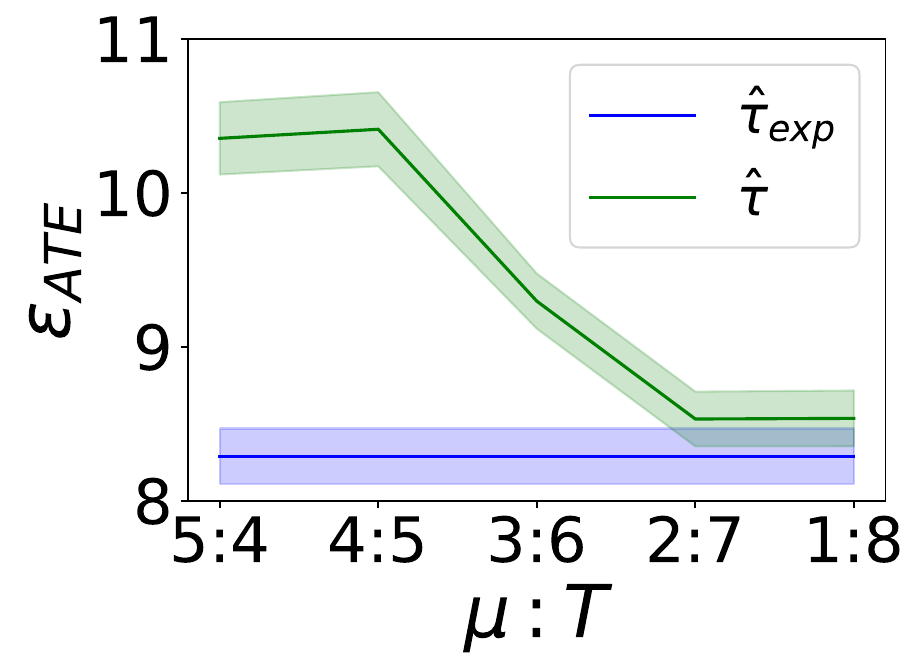}
        \label{fig: MAE varymu}
    }
    \vspace{-.1cm}
    \caption{Results of control experiments with fixed $T$ and varying $\mu$ (Figures \ref{fig: PEHE varyT} and \ref{fig: MAE varyT}), and with fixed $\mu$ and varying $T$ (Figures \ref{fig: PEHE varymu} and \ref{fig: MAE varymu}).}
    \label{fig: varying t mu}
    \vspace{-.3cm}
\end{figure}

\textbf{Vary $T$ and $\mu$}
We conduct control experiments with varying $T$ and varying $\mu$ separately.
First, Figures \ref{fig: PEHE varyT} and \ref{fig: MAE varyT}) show the results for the experiments with fixed $T$ and varying $\mu$. 
As expected, the estimation errors on both heterogeneous and average effect estimations increase as the long-term horizon $\mu$ grows. 
This is primarily due to the increasing estimation errors of $\hat f^{\mu}(\x)$, as established in Lemma \ref{lemma: convergence of f}.
Second, Figures \ref{fig: PEHE varymu} and \ref{fig: MAE varymu} show the results for the experiments with fixed $\mu$ and varying $T$.
As the observed duration $T$ increases, the estimation errors decrease rapidly.
It is reasonable since the larger $T$ results in a faster rate of $\hat \tau(\x)$.
Both findings further support our Lemma \ref{lemma: convergence of f} and Theorem \ref{theo: convergence} again.
Moreover, we observe that, with smaller $\mu$, our estimator closely matches the idealized estimator $\hat \tau_{exp}(\x)$, showing the effectiveness of our estimator.

\begin{wrapfigure}{r}{0.5\textwidth} 
% \begin{figure}[!h]
    \vspace{-.5cm}
    \centering
    \subfigure[Heterogeneous effect estimation error.]
    {\includegraphics[width=.23\textwidth]{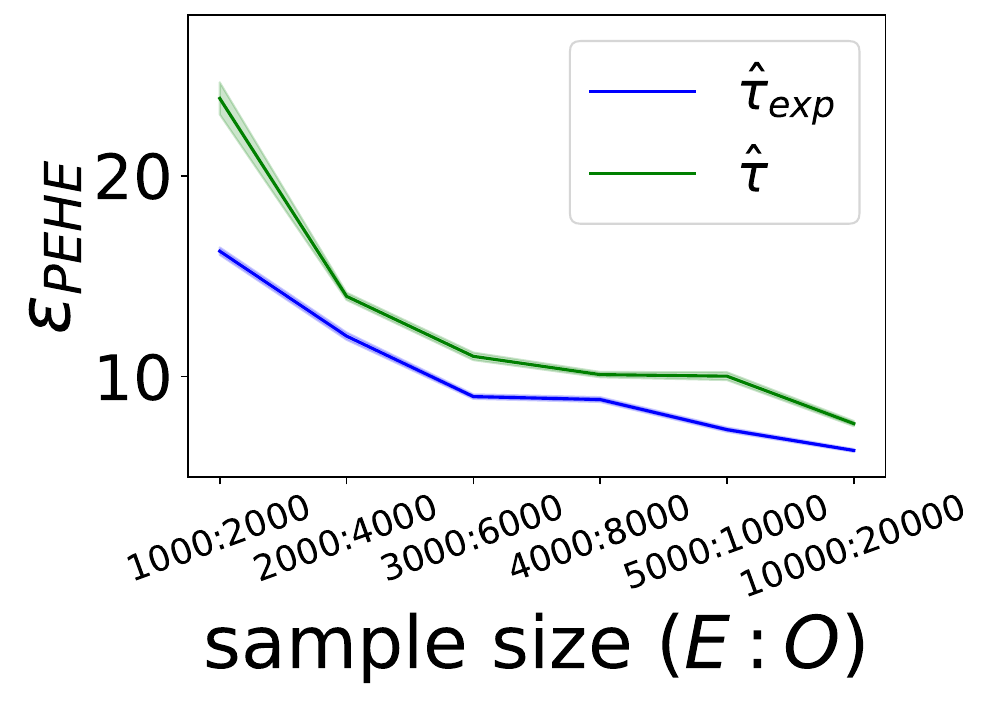}
     \label{fig: PEHE_size}
    }
    \hspace{2pt}
    \subfigure[Average effect estimation error.]
    {\includegraphics[width=.23\textwidth]{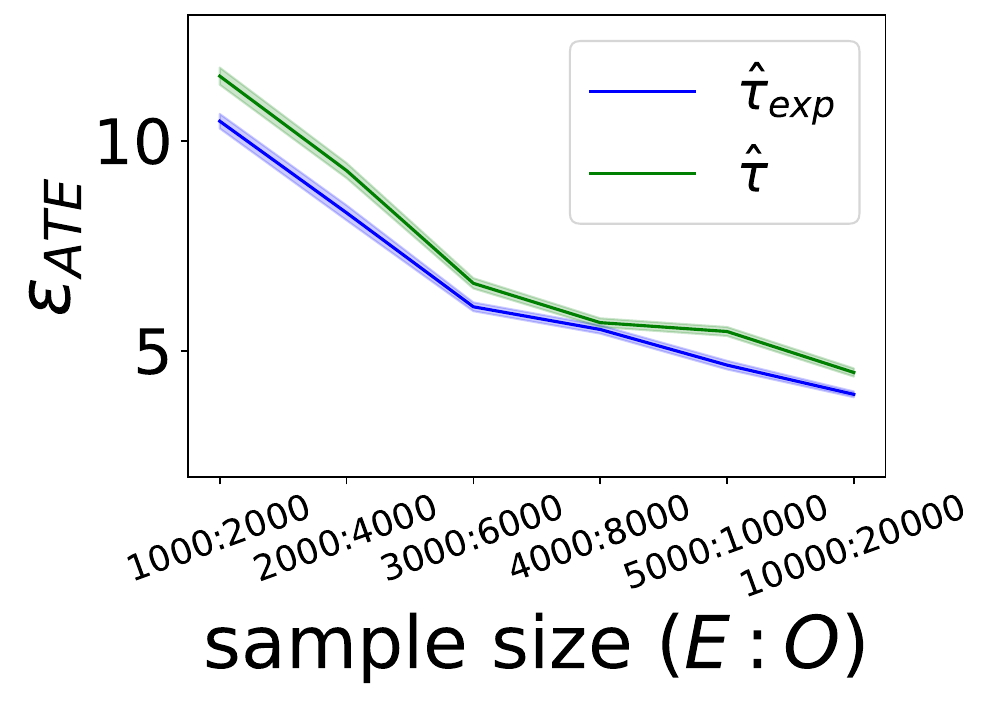}
        \label{fig: MAE_size}
    }
        \vspace{-.2cm}
    \caption{Results of control experiments with varying sample sizes $n_e$ and $n_o$.}
    \label{fig: sample sensitivity}
    \vspace{-.3cm}
% \end{figure}
\end{wrapfigure}
\textbf{Sample size sensitivity}
We conduct control experiments by varying the sample sizes of both experimental and observational data.
The results are presented in Figure \ref{fig: sample sensitivity}.
As expected, the estimator's performance improves as the sample sizes $n_o$ and $n_e$ increase.
Notably, when $n_o$ and $n_e$ become larger, our estimator $\hat \tau(\x)$ closely approaches the idealized estimator $\hat \tau_{exp}(\x)$, which again confirms the effectiveness of our estimator.
Furthermore, when the sample size $n_e \geq 3000$ ($n_o \geq 6000$), the estimation errors become stable, indicating an appropriate sample size threshold.

\section{Conclusion}
\label{Conclusion}

In this paper, we extend the existing CAECB assumption \cite{ghassami2022combining} for long-term causal inference.
The original assumption is restricted to settings where short- and long-term outcomes share the same scale. 
To address this limitation, we introduce a more general assumption, the functional CAECB assumption, which accommodates temporal short-term outcomes.
Under this new assumption, we theoretically establish identification of heterogeneous long-term causal effects and practically develop an estimator by modeling sequential latent confounding.
We further provide a comprehensive theoretical analysis of the estimator.
Experimental results support both the theoretical findings and the estimator’s effectiveness.

\section{Acknowledgments}

This research was supported in part by National Science and
Technology Major Project (2021ZD0111501), National Science Fund for Excellent Young Scholars (62122022), Natural Science Foundation of China (U24A20233, 62206064,
62206061, 62476163, 62406078), Guangdong Basic and Applied Basic Research Foundation (2023B1515120020), and
CCF-DiDi GAIA Collaborative Research Funds (CCF-DiDi GAIA 202403).

% \begin{ack}
% Use unnumbered first level headings for the acknowledgments. All acknowledgments

% Do {\bf not} include this section in the anonymized submission, only in the final paper. You can use the \texttt{ack} environment provided in the style file to automatically hide this section in the anonymized submission.
% \end{ack}

\bibliography{neurips_2025}
\bibliographystyle{plainnat}

% \section*{References}

% References follow the acknowledgments in the camera-ready paper. Use unnumbered first-level heading for
% the references. Any choice of citation style is acceptable as long as you are
% consistent. It is permissible to reduce the font size to \verb+small+ (9 point)
% when listing the references.
% Note that the Reference section does not count towards the page limit.
% \medskip

% {
% \small

% [1] Alexander, J.A.\ \& Mozer, M.C.\ (1995) Template-based algorithms for
% connectionist rule extraction. In G.\ Tesauro, D.S.\ Touretzky and T.K.\ Leen
% (eds.), {\it Advances in Neural Information Processing Systems 7},
% pp.\ 609--616. Cambridge, MA: MIT Press.

% [2] Bower, J.M.\ \& Beeman, D.\ (1995) {\it The Book of GENESIS: Exploring
%   Realistic Neural Models with the GEneral NEural SImulation System.}  New York:
% TELOS/Springer--Verlag.

% [3] Hasselmo, M.E., Schnell, E.\ \& Barkai, E.\ (1995) Dynamics of learning and
% recall at excitatory recurrent synapses and cholinergic modulation in rat
% hippocampal region CA3. {\it Journal of Neuroscience} {\bf 15}(7):5249-5262.
% }

%%%%%%%%%%%%%%%%%%%%%%%%%%%%%%%%%%%%%%%%%%%%%%%%%%%%%%%%%%%%

\appendix

\newpage

  \textit{\large Supplement to}\\ \ \\
      % {\Large \bf ``\texorpdfstring{\raisebox{-1mm}{\includegraphics[scale=0.45]{figures/icon.pdf}}}{NSCtrl}: Causal Temporal Representation Learning with Nonstationary Dynamics''}\
      {\Large \bf ``Long-term Causal Inference via Modeling Sequential
Latent Confounding''}\
    \newcommand{\beginsupplement}{%
% 	\setcounter{table}{0}
% 	\renewcommand{\thetable}{A\arabic{table}}%
% 	\setcounter{figure}{0}
% 	\renewcommand{\thefigure}{A\arabic{figure}}%
% %	\setcounter{algorithm}{0}
% %	\renewcommand{\thealgorithm}{S\arabic{algorithm}}%
% 	\setcounter{section}{0}
% 	\renewcommand{\thesection}{A\arabic{section}}%
% 	\setcounter{theorem}{0}
% 	\renewcommand{\thetheorem}{A\arabic{theorem}}%
% 	\setcounter{proposition}{0}
% 	\renewcommand{\theproposition}{A\arabic{proposition}}%
	% \setcounter{corollary}{0}
	% \renewcommand{\thecorollary}{A\arabic{corollary}}%	
	% \setcounter{lemma}{0}
	% \renewcommand{\thelemma}{A\arabic{lemma}}%
}

\beginsupplement
{\large Appendix organization:}

\DoToC

\section{Assumption Discussion:
Latent Unconfoundedness \cite{athey2020combining}, 
Sequential Outcomes \cite{imbens2022long},
CAECB \cite{ghassami2022combining}, FCAECB}
\label{app: discuss assums}

In this section, we compare different key assumptions in \cite{athey2020combining,imbens2022long,ghassami2022combining} and our FCAECB assumption, and show the differences in terms of causal graphs.

\begin{table}[!h]
    \centering
    \resizebox{1.\textwidth}{!}{ % 自动调整到页面宽度
    \begin{tabular}{c|c|c}
    % \begin{tabular}{m{0.16\textwidth}|m{0.5\textwidth}|m{0.39\textwidth}}
        \toprule
        \Large  Assumptions & \Large Statements  & \Large  Causal graphs \\
          \midrule
        \Large Latent Unconfoundedness \cite{athey2020combining} 
         &  \makecell{ \Large$Y(a) \Vbar A \mid \X, \S(a), G=O$  \\
        }
        &  \makecell{\includegraphics[width=.7\linewidth]{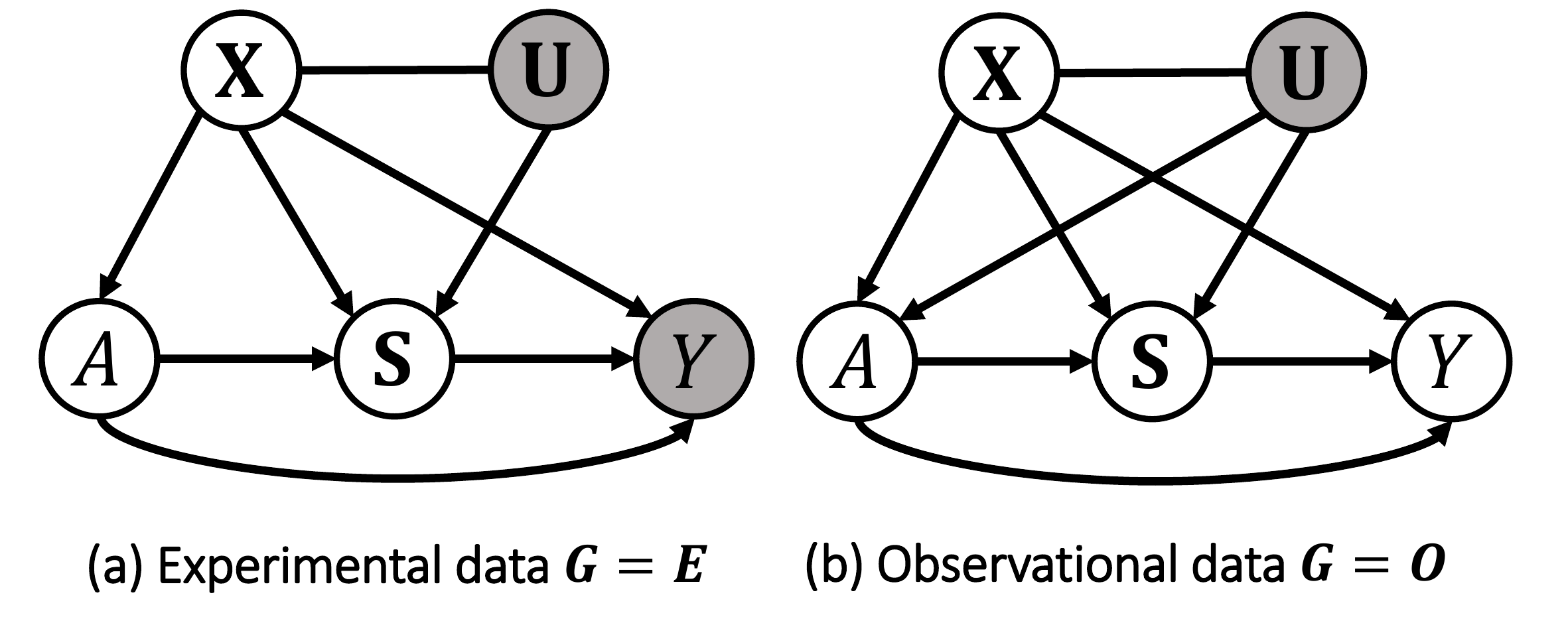}}
         \\ \midrule
        \Large Sequential Outcomes \cite{imbens2022long} 
        &  \makecell{ \Large$(Y(a),\S_3(a)) \Vbar \S_1(a) \mid \S_2(a), \U,\X, G=O$ \\}
        &  \makecell{\includegraphics[width=.7\linewidth]{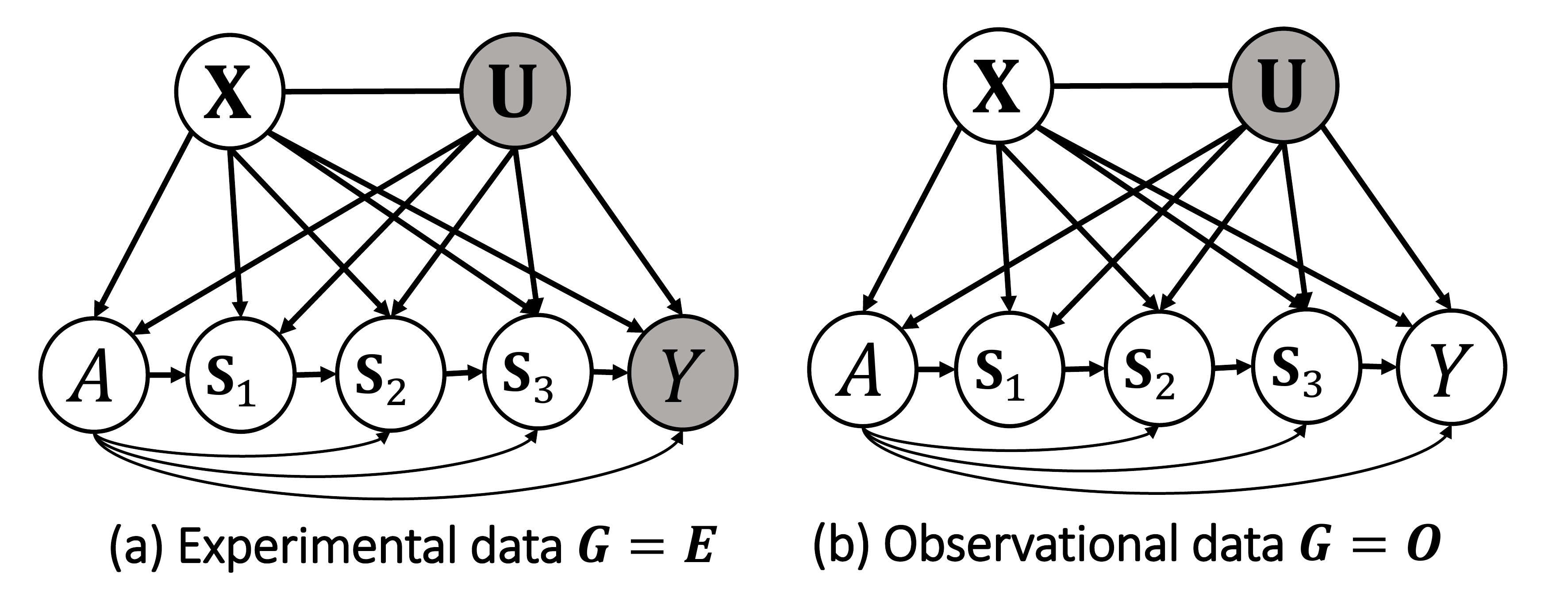}}
        \\ \midrule
        \Large CAECB \cite{ghassami2022combining} 
        & \makecell{  \Large$\bE [S(a)\mid \X,A=0,G=O]-\bE[S(a)\mid \X,A=1,G=O] $
        \\ \\  \Large$ = \bE [Y(a)\mid \X,A=0,G=O]-\bE[Y(a)\mid \X,A=1,G=O] $
        \\}
        &  \makecell{\includegraphics[width=.7\linewidth]{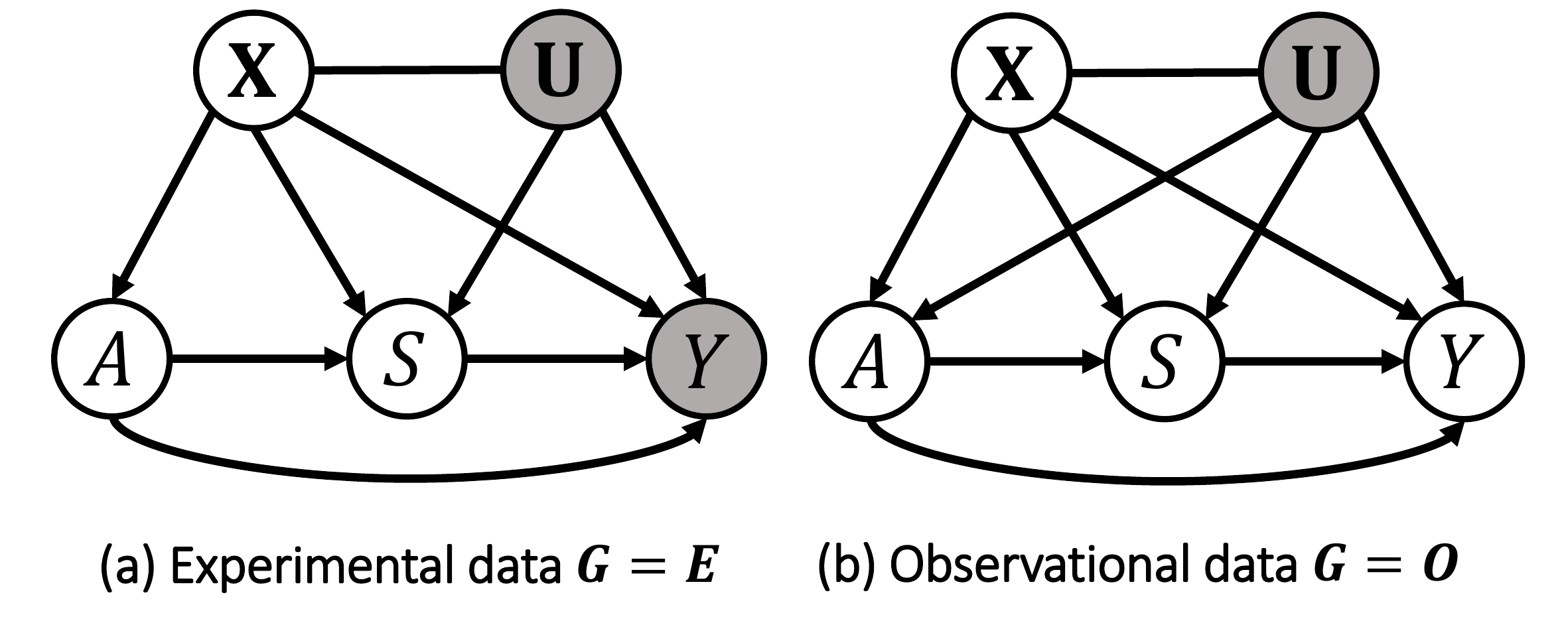}}
        \\ \midrule
        \Large Ours FCAECB& 
        \makecell{ \Large$b_{t+1}(a,\X) = f(\X) b_t(a,\X)$ 
        \\ \\
        where \Large $ b_t(a,\X) = \bE [S_t(a)\mid \X,A=0,G=O]$
        \\ \\
       \hspace{3.5cm}
       \Large  $-\bE[S_t(a)\mid \X,A=1,G=O]$
        }
        &  \makecell{\includegraphics[width=.8\linewidth]{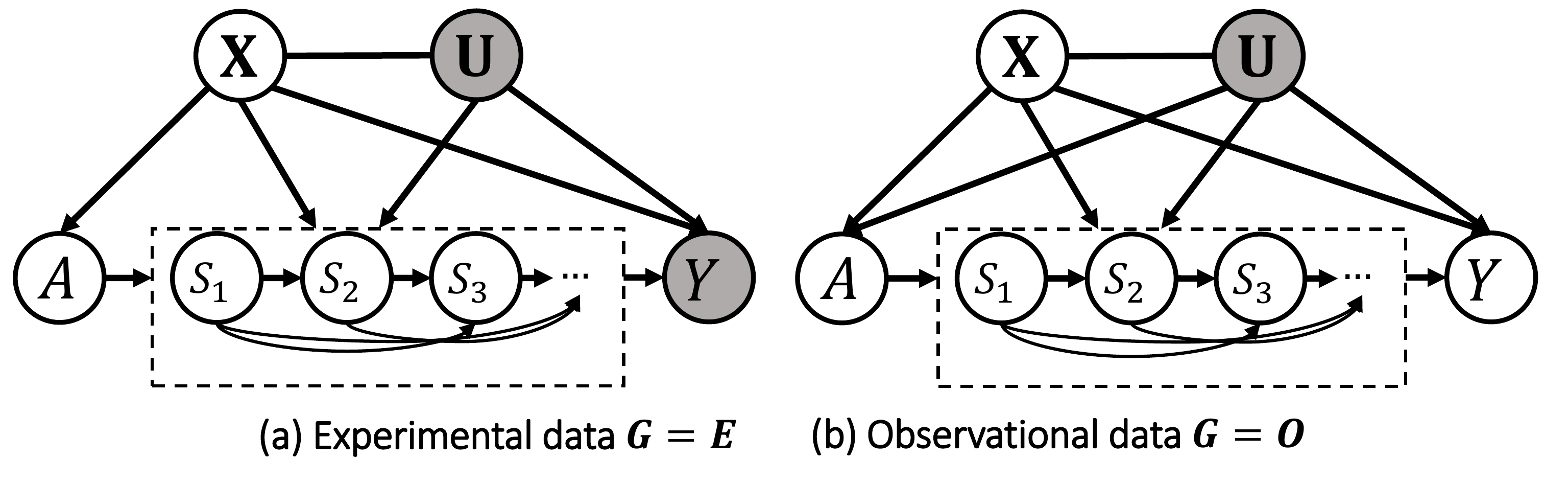}
        }
        \\ \midrule
          \multicolumn{3}{c}{
         \includegraphics[width=1.\linewidth]{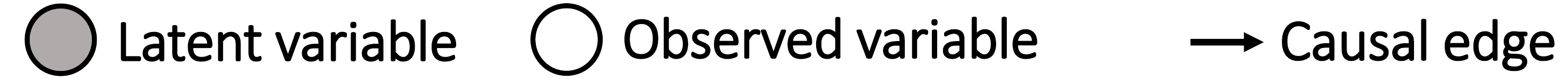}
         }
        \\  \bottomrule
    \end{tabular}
    }
    \caption{Various assumptions and corresponding causal graphs. 
    % Here $b_t(a,\X)=\bE [S_t(a)\mid \X,A=0,G=O]-\bE[S_t(a)\mid \X,A=1,G=O]$.
    }
    \label{tab:my_label}
\end{table}

The Latent Unconfoundedness assumption states that unobserved confounders affect the long-term outcome only through short-term outcomes. 
In terms of causal graphs, $\S(a)$ fully mediates the path from $A$ to $Y(a)$, which results in ruling out the causal edge from $\U$ to $Y$ (Also see the discussion in \cite{ghassami2022combining}).
This assumption ensures that $\S$ can serve a role of $\U$, making the unconfoundedness assumption hold in a certain way and thus identifying long-term effects.
Instead, the Sequential Outcomes assumption and CAECB assumption allow for such a causal edge $\U \rightarrow Y$, while they posit different requirements.
The Sequential Outcomes assumption requires that the short- and long-term outcomes can only be directly affected by the outcomes immediately preceding them, but not outcomes further in the past. 
In other words, $\S_{t+1}$ can only be affected by $\S_{t}$, but not by $\S_{ t-1}$ or earlier.
This assumption ensures that $\S_1$ and $\S_3$ act similarly as negative controls in proximal causal learning \cite{miao2018ident,tchetgen2024introduction}, making effects identifiable.
The CAECB assumption does not posit any requirement on the causal graph, but requires that the short- and long-term outcomes are measured on the same scale.
Thus, the short-term outcome becomes a scale variable and may be easily violated when many short-term outcomes can be observed.
This assumption states that the confounding biases in short and long term are the same, making long-term confounding bias can be identified using the observed short-term outcomes.
Our FCAECB assumption likewise places no graph restriction, and further, we allow for temporal short-term outcomes instead of a single $S$.
The FCAECB assumption posits that the sequence of short‑term confounding biases follows a learnable temporal trend, which enables us to infer the long‑term bias from observed $\S_{1:T}$.

Compared with the Latent Unconfoundedness and Sequential Outcomes assumptions, our proposed FCAECB assumption is neither logically stronger nor weaker. 
While those two assumptions impose graphical constraints, FCAECB instead posits a particular form of data-generating process.
And compared to the CAECB assumption, the FCAECB assumption is weaker, as CAECB corresponds to a special case of FCAECB where $f(\X)=1$.

Our FCAECB assumption can be further relaxed by allowing it to vary over time. In other words, we can replace $f(\X)$ with $f(\X,t)$ without any other modification, and the identification result also holds with almost the same proofs.
However, this generalization introduces a new challenge: how to extrapolate $f(\X,t)$, learned with $t \in \{1,...,T\}$, to $f(\X,T+\mu)$.
Addressing this challenge requires additional assumptions on $f$, such as temporal smoothness, autoregressive dependencies, \textit{etc}.
We leave this extension for future work.

\section{Broader impacts}
\label{app: broader impact}
Our work focuses on long-term causal inference, whose goal is to estimate long-term causal effects by combining long-term observational data and short-term experimental data. Our proposed method could be applied to a wide range of applications, such as decision-making in marketing and preventive measure studies in epidemiology. There are no specific possible harms or negative societal impacts that we should claim.

\section{Limitations}
\label{app: limitation}
While our work relaxes some of the existing assumptions, it still relies on relatively strong ones.
Regarding data requirements, similar to \cite{athey2020combining,ghassami2022combining,imbens2022long}, our work depends on access to short-term experimental data, which may not be feasible in certain contexts, such as studies on smoking where interventions would be unethical. Additionally, the core Assumption \ref{assum: time series equ bias} proposed in this paper may not hold in practice and remains untestable due to the absence of long-term outcomes in experimental data, although we propose a partial validation using short-term outcomes.

\section{Additional Experimental Details and Results}
\label{app: add exp}

\textbf{Data Generation Process}
Our data generation process partly follows \cite{kallus2018removing}.
Specifically, we first generate the treatments as follows: $A\mid G=O \sim \text{Bernoulli}(0.6)$ and $A\mid G=E \sim \text{Bernoulli}(0.4)$. Then we generate the observed $\X$ and the unobserved $\U$ as follows:
   \begin{equation} 
    \begin{aligned}
       & (\X,\U)\mid A,G=E \sim \mathcal N ([\frac{2A-1}{2},0], 
       \begin{bmatrix} 1 & 0 \\ 0 & 1
       \end{bmatrix})
       \\  
       & (\X,\U)\mid A,G=O \sim \mathcal N ([\frac{1-2A}{2},0], 
       \begin{bmatrix} 1 & A-0.5 \\ A-0.5 & 1
       \end{bmatrix}).
    \end{aligned}
    \end{equation}
%which effectively introduces the unobserved confounding in the observational group.
Finally, we generate the $T$-step short-term outcomes $\S$ and the long-term outcome $Y=S_{T+\mu}$ satisfying  Assumption \ref{assum: time series equ bias} as follows:
   \begin{equation} 
    \begin{aligned}
       % & S_t(1) = 1 + 1.1 \X + \U + \sum_{k=0}^t S_k(1) + \epsilon_{S_t},
       % \\&
       % S_t(0) =  \X + \U + \sum_{k=0}^t S_k(0) + \epsilon_{S_t},
       S_t(A) = A + 0.1A\X + \X + \U + \sum_{k=0}^t S_k(A) + \epsilon_{S_t(A)}
    \end{aligned}
    \end{equation}
where $\epsilon_{S_t(A)}$ are  Gaussian noises. This will result in non-equal confounding bias in different time steps $t$, i.e., $\omega_t(\x)=-2^t \x$. 

\textbf{Baselines and implementation details}
We compare our method with the following baselines:
\begin{itemize}
    \item LTEE \cite{cheng2021long}. LTEE assumes the unconfoundedness assumption, and it uses a dual-headed RNN to capture the relationships of time-dependent surrogates and then estimates the effects using the predicted long-term outcomes. We use the code available at \url{https://github.com/GitHubLuCheng/LTEE}.
    % \item Surrogate Index \cite{athey2019surrogate}. Surrogate Index assumes the surrogacy assumption, and it 
    \item Latent Unconfoundedness \cite{athey2020combining}. \citet{athey2020combining} propose a long-term effect estimator under the latent unconfoundedness assumption. We implement this method following Sec. 4.1 in \cite{athey2020combining} using correctly specified models.
    \item Sequential Outcomes \cite{imbens2022long}. \citet{imbens2022long} propose a long-term effect estimator by treating temporal short-term outcomes as proxies. 
    we test three grouping strategies: (1) $S_{\lceil T/2 \rceil -1}, S_{\lceil T/2 \rceil}, S_{\lceil T/2 \rceil +1}$, (2) $S_1, S_{\lceil T/2 \rceil}, S_T$, and (3) $S_{1:\lceil T/2 \rceil-1}, S_{\lceil T/2 \rceil}, S_{\lceil T/2 \rceil+1:T}$.
    We use the code available at \url{https://github.com/CausalML/LongTermCausalInference/tree/main}.
    \item NNPIV \cite{meza2021nested}. \citet{meza2021nested} propose an estimator for nested nonparametric instrumental variable regression, which is then applied to estimate the long-term heterogeneous effect under latent unconfoundedness and surrogacy assumption. We denote the estimators as NNPIV\_LU and NNPIV\_SInd under the latent unconfoundedness assumption \cite{athey2020combining} and surrogacy assumption \cite{athey2019surrogate}, respectively. We use the code available at \url{https://github.com/isaacmeza/NNPIV}.
    % \item NNPIV\_SInd \cite{meza2021nested}. We use the code available at \url{https://github.com/isaacmeza/NNPIV}.
    \item CAECB \cite{ghassami2022combining}. We implement CAECB \cite{ghassami2022combining} following the identification result in Eq. \ref{eq: identi}. Since this estimator requires only using one short-term outcome, we implement four versions, including using the first, middle, last, and random one short-term outcome among all short-term outcomes.
    \item T-learner \cite{kunzel2019metalearners} trained using observational $Y$. This estimator is biased due to the latent confounding. We implement it using linear regression.
    \item T-learner \cite{kunzel2019metalearners} trained using experimental $Y$. We denote this estimator as $\hat{\tau}_{exp}$. This estimator is idealized and infeasible in practice since $Y$ is missing in the experimental data and serves as an idealized benchmark. We implement it using linear regression.
    \item Ours $\hat \tau$. We implement our method according to the algorithm described in Sec. \ref{sec: alg} using correctly specified models. Unless otherwise specified, we train the model using all short-term outcomes and omit step S0.
\end{itemize}

All experiments are conducted on a machine equipped with an Intel(R) Xeon(R) E5-2620 v4 @ 2.10GHz CPU and a GeForce RTX 2080 Ti GPU. Each reported result is averaged over 50 runs.

\textbf{Full Results of Table \ref{table:vary mu}}.
We present the full result of Table \ref{table:vary mu} in Tables \ref{table:vary mu whole} and \ref{table:vary T whole}. The results are expected and consistent with the analysis presented in the main text. Our estimator outperforms all baselines and closely aligns with the idealized estimator due to the correct assumption.

\begin{table}[t]
\caption{The results on synthetic data are presented as the mean and standard deviation. Lower values indicate better performance. The best-performing methods are bolded, with the exception of the idealized estimator $\hat{\tau}_{exp}$, which is excluded from the comparison.}
\label{sample-table}
% \vspace{-.2cm}
\begin{center}
\resizebox{\textwidth}{!}{ % 自动调整到页面宽度
\begin{tabular}{l|cc|cc|cc|cc|cc}
\toprule
Data 
& \multicolumn{2}{|c|}{$\mu=1$, $T=6$} 
& \multicolumn{2}{|c|}{$\mu=2$, $T=6$} 
& \multicolumn{2}{|c|}{$\mu=3$, $T=6$} 
& \multicolumn{2}{|c|}{$\mu=4$, $T=6$} 
& \multicolumn{2}{|c}{$\mu=5$, $T=6$} 
\\
\midrule
Metrics
& \multicolumn{1}{|c}{${\varepsilon_{PEHE}}$} 
& \multicolumn{1}{c|}{$\varepsilon_{ATE}$} 
& \multicolumn{1}{|c}{${\varepsilon_{PEHE}}$} 
& \multicolumn{1}{c|}{$\varepsilon_{ATE}$} 
& \multicolumn{1}{|c}{${\varepsilon_{PEHE}}$} 
& \multicolumn{1}{c}{$\varepsilon_{ATE}$} 
& \multicolumn{1}{|c}{${\varepsilon_{PEHE}}$} 
& \multicolumn{1}{c|}{$\varepsilon_{ATE}$} 
& \multicolumn{1}{|c}{${\varepsilon_{PEHE}}$} 
& \multicolumn{1}{c|}{$\varepsilon_{ATE}$} 
\\
\midrule
$\hat{\tau}_{exp}$ (idealized)
& 3.0033 $\pm$ 1.8226 & 2.0722 $\pm$ 1.7681
& 6.0076 $\pm$ 3.6446 & 4.1457 $\pm$ 3.5353
& 12.0142 $\pm$ 7.2895 & 8.2901 $\pm$ 7.0709
& 24.0297 $\pm$ 14.5790 & 16.5806 $\pm$ 14.1423
& 48.0590 $\pm$ 29.1580 & 33.1620 $\pm$ 28.2845  \\
\hline
T-learner (using obs. $Y$)
& 71.8431 $\pm$ 3.1612 & 6.8211 $\pm$ 2.8357
& 143.6844 $\pm$ 6.3217 & 13.6400 $\pm$ 5.6710
& 287.3693 $\pm$ 12.6434 & 27.2795 $\pm$ 11.3424
& 574.7384 $\pm$ 25.2877 & 54.5614 $\pm$ 22.6837
& 1149.4769 $\pm$ 50.5746 & 109.1205 $\pm$ 45.3685   \\
LTEE
& 73.0421 $\pm$ 7.1531 & 10.0977 $\pm$ 6.1071
& 142.1183 $\pm$ 10.9446 & 17.3357 $\pm$ 10.0487
& 287.3264 $\pm$ 28.5202 & 32.4417 $\pm$ 19.5739
& 589.0638 $\pm$ 56.6554 & 79.3176 $\pm$ 41.2797
& 1168.7612 $\pm$ 122.7687 & 137.0583 $\pm$ 107.5567   \\
% Surrogate Index
% & - & 2.9972 $\pm$ 2.3146
% & - & 5.9945 $\pm$ 4.6286
% & - & 11.9888 $\pm$ 9.2576
% & - & 23.9772 $\pm$ 18.5151
% & - & 47.9542 $\pm$ 37.0302   \\
Latent Unconfoundedness
& - & 2.6322 $\pm$ 2.0120
& - & 5.2648 $\pm$ 4.0238
& - & 10.5293 $\pm$ 8.0471
& - & 21.0585 $\pm$ 16.0946
& - & 42.1179 $\pm$ 32.1874   \\
Sequential Outcomes ($S_{\lceil T/2 \rceil -1}, S_{\lceil T/2 \rceil }, S_{\lceil T/2 \rceil +1}$)
& - & 67.5881 $\pm$ 3.0107
& - & 135.1792 $\pm$ 6.0220
& - & 270.3607 $\pm$ 12.0428
& - & 540.7164 $\pm$ 24.0884
& - & 1081.4361 $\pm$ 48.1759    \\
Sequential Outcomes ($S_1, S_{\lceil T/2 \rceil }, S_T$)
& - & 67.5790 $\pm$ 3.0138
& - & 135.1608 $\pm$ 6.0282
& - & 270.3237 $\pm$ 12.0553
& - & 540.6429 $\pm$ 24.1133
& - & 1081.2884 $\pm$ 48.2252
\\
Sequential Outcomes ($S_{1:\lceil T/2 \rceil -1}, S_{\lceil T/2 \rceil }, S_{\lceil T/2 \rceil +1:T}$)
& - & 67.5797 $\pm$ 3.0151
& - & 135.1622 $\pm$ 6.0309
& - & 270.3266 $\pm$ 12.0607
& - & 540.6486 $\pm$ 24.1242
& - & 1081.3001 $\pm$ 48.2472
\\
NNPIV\_LU
& 70.7124 $\pm$ 2.5320 & 6.1302 $\pm$ 2.8888
& 141.4309 $\pm$ 5.0655 & 12.2646 $\pm$ 5.7776
& 282.8595 $\pm$ 10.1320 & 24.5293 $\pm$ 11.5558
& 565.7158 $\pm$ 20.2586 & 49.0542 $\pm$ 23.1130
& 1131.4331 $\pm$ 40.5209 & 98.1118 $\pm$ 46.2263   \\
NNPIV\_SInd
& 70.7155 $\pm$ 2.5329 & 6.1306 $\pm$ 2.8884
& 141.4320 $\pm$ 5.0665 & 12.2626 $\pm$ 5.7782
& 282.8655 $\pm$ 10.1314 & 24.5240 $\pm$ 11.5558
& 565.7307 $\pm$ 20.2627 & 49.0492 $\pm$ 23.1115
& 1131.4602 $\pm$ 40.5252 & 98.0989 $\pm$ 46.2232   \\
CAECB (using first $S$)
& 70.7257 $\pm$ 3.1153 & 6.7160 $\pm$ 2.7984
& 142.5669 $\pm$ 6.2757 & 13.5348 $\pm$ 5.6335
& 286.2518 $\pm$ 12.5973 & 27.1744 $\pm$ 11.3048
& 573.6209 $\pm$ 25.2417 & 54.4563 $\pm$ 22.6461
& 1148.3595 $\pm$ 50.5286 & 109.0154 $\pm$ 45.3308   \\
CAECB (using middle $S$)
& 67.3754 $\pm$ 2.9827 & 6.3982 $\pm$ 2.6884
& 139.2165 $\pm$ 6.1417 & 13.2171 $\pm$ 5.5210
& 282.9014 $\pm$ 12.4626 & 26.8566 $\pm$ 11.1910
& 570.2704 $\pm$ 25.1065 & 54.1385 $\pm$ 22.5317
& 1145.0089 $\pm$ 50.3933 & 108.6976 $\pm$ 45.2162  \\
CAECB (using last $S$)
& 36.1202 $\pm$ 2.0211 & 3.5778 $\pm$ 1.8603
& 107.9443 $\pm$ 4.9841 & 10.2576 $\pm$ 4.6541
& 251.6244 $\pm$ 11.2383 & 23.8972 $\pm$ 10.2153
& 538.9915 $\pm$ 23.8540 & 51.1791 $\pm$ 21.5097
& 1113.7291 $\pm$ 49.1287 & 105.7381 $\pm$ 44.1729  \\
CAECB (using random $S$)
& 59.8772 $\pm$ 11.9929 & 5.7194 $\pm$ 2.6871
& 131.7174 $\pm$ 12.7309 & 12.5383 $\pm$ 5.3841
& 275.4020 $\pm$ 16.0775 & 26.1778 $\pm$ 10.9897
& 562.7700 $\pm$ 26.2888 & 53.4597 $\pm$ 22.2995
& 1137.5094 $\pm$ 50.1542 & 108.0188 $\pm$ 44.9689  \\
Ours  $\hat \tau$
&\textbf{ 3.1110 $\pm$ 1.8230 } & \textbf{ 2.1470 $\pm$ 1.7826} 
& \textbf{ 6.5594 $\pm$ 3.5589 } & \textbf{ 4.4610 $\pm$ 3.5302 }
&\textbf{ 13.9945 $\pm$ 7.1682 } & \textbf{ 9.2961 $\pm$ 7.0820} 
& \textbf{ 29.9866 $\pm$ 15.0884 } & \textbf{ 19.3927 $\pm$ 14.4300 }
&\textbf{ 64.3076 $\pm$ 32.9868 } & \textbf{ 40.5425 $\pm$ 29.7264 }  \\
\bottomrule
\end{tabular}
}
\end{center}
\label{table:vary T whole}
% \vskip -0.1in
\end{table}

\begin{table}[t]
\caption{The results on synthetic data are presented as the mean and standard deviation. Lower values indicate better performance. The best-performing methods are bolded, with the exception of the idealized estimator $\hat{\tau}_{exp}$, which is excluded from the comparison.}
\label{sample-table}
% \vspace{-.5cm}
\begin{center}
\resizebox{\textwidth}{!}{ % 自动调整到页面宽度
\begin{tabular}{l|cc|cc|cc|cc|cc}
\toprule
Data 
& \multicolumn{2}{|c|}{$\mu=5$, $T=4$} 
& \multicolumn{2}{|c|}{$\mu=4$, $T=5$} 
& \multicolumn{2}{|c|}{$\mu=3$, $T=6$} 
& \multicolumn{2}{|c|}{$\mu=2$, $T=7$} 
& \multicolumn{2}{|c}{$\mu=1$, $T=8$} 
\\
\midrule
Metrics
& \multicolumn{1}{|c}{${\varepsilon_{PEHE}}$} 
& \multicolumn{1}{c|}{$\varepsilon_{ATE}$} 
& \multicolumn{1}{|c}{${\varepsilon_{PEHE}}$} 
& \multicolumn{1}{c|}{$\varepsilon_{ATE}$} 
& \multicolumn{1}{|c}{${\varepsilon_{PEHE}}$} 
& \multicolumn{1}{c}{$\varepsilon_{ATE}$} 
& \multicolumn{1}{|c}{${\varepsilon_{PEHE}}$} 
& \multicolumn{1}{c|}{$\varepsilon_{ATE}$} 
& \multicolumn{1}{|c}{${\varepsilon_{PEHE}}$} 
& \multicolumn{1}{c|}{$\varepsilon_{ATE}$} 
\\
\midrule
$\hat{\tau}_{exp}$ (idealized)
& 12.0142 $\pm$ 7.2895 & 8.2901 $\pm$ 7.0709
& 12.0142 $\pm$ 7.2895 & 8.2901 $\pm$ 7.0709
& 12.0142 $\pm$ 7.2895 & 8.2901 $\pm$ 7.0709
& 12.0142 $\pm$ 7.2895 & 8.2901 $\pm$ 7.0709
& 12.0142 $\pm$ 7.2895 & 8.2901 $\pm$ 7.0709 \\
\hline
T-learner (using obs. $Y$)
& 287.3693 $\pm$ 12.6434 & 27.2795 $\pm$ 11.3423
& 287.3693 $\pm$ 12.6434 & 27.2795 $\pm$ 11.3423
& 287.3693 $\pm$ 12.6434 & 27.2795 $\pm$ 11.3423
& 287.3693 $\pm$ 12.6434 & 27.2795 $\pm$ 11.3423
& 287.3693 $\pm$ 12.6434 & 27.2795 $\pm$ 11.3423  \\
LTEE
& 296.1849 $\pm$ 27.7377 & 32.9559 $\pm$ 19.0871
& 288.1059 $\pm$ 31.1373 & 34.7556 $\pm$ 20.5223
& 287.3264 $\pm$ 28.5202 & 32.4417 $\pm$ 19.5739
& 289.8522 $\pm$ 25.2498 & 41.9068 $\pm$ 25.6233
& 284.0581 $\pm$ 26.0439 & 29.2638 $\pm$ 20.0548 \\
% Surrogate Index
% & - & 11.9857 $\pm$ 9.2354
% & - & 11.9916 $\pm$ 9.2497
% & - & 11.9888 $\pm$ 9.2576
% & - & 11.9853 $\pm$ 9.2589
% & - & 11.9851 $\pm$ 9.2587  \\
Latent Unconfoundedness
& - & 10.5248 $\pm$ 8.0405
& - & 10.5302 $\pm$ 8.0423
& - & 10.5293 $\pm$ 8.0471
& - & 10.5270 $\pm$ 8.0476
& - & 10.5270 $\pm$ 8.0465  \\
Sequential Outcomes ($S_{\lceil T/2 \rceil -1}, S_{\lceil T/2 \rceil }, S_{\lceil T/2 \rceil +1}$)
& - & 270.3555 $\pm$ 12.0474
& - & 270.3607 $\pm$ 12.0428
& - & 270.3607 $\pm$ 12.0428
& - & 270.3272 $\pm$ 12.0610
& - & 270.3272 $\pm$ 12.0610  \\
Sequential Outcomes ($S_1, S_{\lceil T/2 \rceil }, S_T$)
& - & 270.3658 $\pm$ 12.0486
& - & 270.3213 $\pm$ 12.0654
& - & 270.3237 $\pm$ 12.0553
& - & 270.3323 $\pm$ 12.0565
& - & 270.3247 $\pm$ 12.0560
\\
Sequential Outcomes ($S_{1:\lceil T/2 \rceil -1}, S_{\lceil T/2 \rceil }, S_{\lceil T/2 \rceil +1:T}$)
& - & 270.3554 $\pm$ 12.0473
& - & 270.3266 $\pm$ 12.0607
& - & 270.3266 $\pm$ 12.0607
& - & 270.3295 $\pm$ 12.0557
& - & 270.3295 $\pm$ 12.0557
\\
NNPIV\_LU
& 282.8435 $\pm$ 10.1334 & 24.6196 $\pm$ 11.5648
& 282.8424 $\pm$ 10.1261 & 24.5424 $\pm$ 11.5495
& 282.8595 $\pm$ 10.1320 & 24.5293 $\pm$ 11.5558
& 282.8690 $\pm$ 10.1315 & 24.5386 $\pm$ 11.5653
& 282.8671 $\pm$ 10.1264 & 24.5355 $\pm$ 11.5659  \\
NNPIV\_SInd
& 282.8679 $\pm$ 10.1239 & 24.5522 $\pm$ 11.5413
& 282.8408 $\pm$ 10.1337 & 24.5238 $\pm$ 11.5430
& 282.8655 $\pm$ 10.1314 & 24.5240 $\pm$ 11.5558
& 282.8683 $\pm$ 10.1288 & 24.5331 $\pm$ 11.5663
& 282.8679 $\pm$ 10.1240 & 24.5326 $\pm$ 11.5657  \\
CAECB (using first $S$)
& 286.2518 $\pm$ 12.5973 & 27.1744 $\pm$ 11.3048
& 286.2518 $\pm$ 12.5973 & 27.1744 $\pm$ 11.3048
& 286.2518 $\pm$ 12.5973 & 27.1744 $\pm$ 11.3048
& 286.2519 $\pm$ 12.5973 & 27.1744 $\pm$ 11.3048
& 286.2519 $\pm$ 12.5973 & 27.1744 $\pm$ 11.3048 \\
CAECB (using middle $S$)
& 285.1353 $\pm$ 12.5525 & 27.0677 $\pm$ 11.2662
& 282.9014 $\pm$ 12.4625 & 26.8566 $\pm$ 11.1910
& 282.9014 $\pm$ 12.4626 & 26.8566 $\pm$ 11.1910
& 278.4314 $\pm$ 12.2812 & 26.4329 $\pm$ 11.0441
& 278.4314 $\pm$ 12.2812 & 26.4330 $\pm$ 11.0441  \\
CAECB (using last $S$)
& 278.4314 $\pm$ 12.2812 & 26.4330 $\pm$ 11.0441
& 269.4957 $\pm$ 11.9268 & 25.5886 $\pm$ 10.7549
& 251.6244 $\pm$ 11.2383 & 23.8972 $\pm$ 10.2153
& 215.8874 $\pm$ 9.9675 & 20.5128 $\pm$ 9.3073
& 144.4755 $\pm$ 8.0850 & 14.3059 $\pm$ 7.4382 \\
CAECB (using random $S$)
& 283.3244 $\pm$ 13.2920 & 26.8862 $\pm$ 11.1734
& 280.1459 $\pm$ 13.3821 & 26.5476 $\pm$ 11.0112
& 275.4020 $\pm$ 16.0775 & 26.1778 $\pm$ 10.9897
& 270.3812 $\pm$ 22.4630 & 25.5679 $\pm$ 10.7551
& 254.5652 $\pm$ 45.7734 & 24.0325 $\pm$ 10.6828 \\
Ours  $\hat \tau$
&\textbf{ 19.1631 $\pm$ 12.1030 } &\textbf{ 10.3532 $\pm$ 9.2908} 
&\textbf{ 16.6761 $\pm$ 13.1391 } &\textbf{ 10.4123 $\pm$ 9.5082} 
&\textbf{ 13.9945 $\pm$ 7.1682 } &\textbf{ 9.2961 $\pm$ 7.0820} 
&\textbf{ 12.7653 $\pm$ 7.4767 } &\textbf{ 8.5308 $\pm$ 7.0285} 
&\textbf{ 12.3040 $\pm$ 7.3958 } &\textbf{ 8.5348 $\pm$ 7.1581}   \\
\bottomrule
\end{tabular}
}
\end{center}
\label{table:vary mu whole}
% \vskip -0.1in
\end{table}

\section{Proof of Theorem \ref{theo: identifi}}\label{app: proof of iden}
We first restate Theorem \ref{theo: identifi} as follows:
% We provide the proof Theorem \ref{theo: identify  }.
\begin{theorem} 
Suppose Assumptions \ref{assum: consist}, \ref{assum: positi}, \ref{assum: internal validity of obs}, \ref{assum: internal validity of exp}, \ref{assum: external validity of exp} and \ref{assum: equ bias} hold, then $\tau(\x)$ can be identified as follows:
    \begin{equation}
      \begin{aligned}
       & \tau(\x) \\
            = & \mathbb E[Y(1) - Y(0)\mid \X=\x] \\
            = & \mu_Y^O(1,\x) - \mu_Y^O(0,\x)  
             +  \mu_S^E(1,\x) -  \mu_S^E(0,\x)
             +  \mu_S^O(0,\x) - \mu_S^O(1,\x) .
      \end{aligned}
    \end{equation}
\end{theorem}

\begin{proof}
    \begin{equation} \label{eq: tau decom} 
        \begin{aligned}
            & \tau(\X)  \\
            = & \mathbb E[Y(1)\mid \X,G=O] - \mathbb E[Y(0)\mid \X,G=O] \\
            = & \mathbb E[Y(1)\mid \X,G=O,A=1]p(A=1\mid \X,G=O) 
                + \mathbb E[Y(1)\mid \X,G=O,A=0]p(A=0\mid \X,G=O) \\
              & - \mathbb E[Y(0)\mid \X,G=O,A=1]p(A=1\mid \X,G=O) 
                 + \mathbb E[Y(0)\mid \X,G=O,A=0]p(A=0\mid \X,G=O) \\
            = & \mathbb E[Y(1)\mid \X,G=O,A=1]p(A=1\mid \X,G=O) 
                + \mathbb E[Y(1)\mid \X,G=O,A=0]p(A=0\mid \X,G=O) \\
              & - \mathbb E[Y(0)\mid \X,G=O,A=1]p(A=1\mid \X,G=O) 
                + \mathbb E[Y(0)\mid \X,G=O,A=0]p(A=0\mid \X,G=O) \\
              & + \mathbb E[Y(1)\mid \X,G=O,A=1]p(A=0\mid \X,G=O)
                 - \mathbb E[Y(1)\mid \X,G=O,A=1]p(A=0\mid \X,G=O) \\
              & + \mathbb E[Y(0)\mid \X,G=O,A=0]p(A=1\mid \X,G=O)
                - \mathbb E[Y(0)\mid \X,G=O,A=0]p(A=1\mid \X,G=O) \\
            = & \mathbb E[Y(1)\mid \X,G=O,A=1] - \mathbb E[Y(0)\mid \X,G=O,A=0] \\
            & + \{ \mathbb E[Y(1)\mid \X,G=O,A=0] - \mathbb E[Y(1)\mid \X,G=O,A=1]\}  \times p(A=0\mid \X,G=O) \\
              & + \{ \mathbb E[Y(0)\mid \X,G=O,A=0] - \mathbb E[Y(0)\mid \X,G=O,A=1]\} \times p(A=1\mid \X,G=O) \\
            = & \mathbb E[Y\mid \X,G=O,A=1] - \mathbb E[Y\mid \X,G=O,A=0] \\
              & + \{ \mathbb E\left[ S(1) \mid \X,G=O,A=0\right] - \mathbb E\left[ S(1)  \mid \X,G=O,A=1\right] \} \times p(A=0\mid \X,G=O) \\
              & + \{ \mathbb E\left[ S(0) \mid \X,G=O,A=0\right] - \mathbb E\left[ S(0) \mid \X,G=O,A=1\right]  \} \times p(A=1\mid \X,G=O) , 
        \end{aligned}
    \end{equation}
    where the first equality is based on Assumption \ref{assum: external validity of exp} and the last equality is based on Assumption \ref{assum: equ bias}.
    Similarly, for short-term conditional causal effects, we have:
    \begin{equation} \label{eq: stce decom}
        \begin{aligned}
            & \mathbb E[S(1)\mid \X,G=O] - \mathbb E[S(0)\mid \X,G=O] \\
            = & \mathbb E[S\mid \X,G=O,A=1] - \mathbb E[S\mid \X,G=O,A=0] \\
              & + \{ \mathbb E[S(1)\mid \X,G=O,A=0] - \mathbb E[S(1)\mid \X,G=O,A=1]\} \times p(A=0\mid \X,G=O) \\
              & + \{ \mathbb E[S(0)\mid \X,G=O,A=0] - \mathbb E[S(0)\mid \X,G=O,A=1]\} \times p(A=1\mid \X,G=O) \\
        \end{aligned}
    \end{equation}
    Then, combining Eq. \eqref{eq: tau decom} and \eqref{eq: stce decom}, we have
    \begin{equation} 
        \begin{aligned}
            & \tau(\X) \\
            = & \mathbb E[Y(1)\mid \X,G=O] - \mathbb E[Y(0)\mid \X,G=O] \\
            = & \mathbb E[Y\mid \X,G=O,A=1] - \mathbb E[Y\mid \X,G=O,A=0] \\
              & + \mathbb E[S(1)\mid \X,G=O] - \mathbb E[S(0)\mid \X,G=O] 
              - \mathbb E[S\mid \X,G=O,A=1] + \mathbb E[S\mid \X,G=O,A=0] \\
            = & \mathbb E[Y\mid \X,G=O,A=1] - \mathbb E[Y\mid \X,G=O,A=0] \\
              & + \mathbb E[S(1)\mid \X,G=E] - \mathbb E[S(0)\mid \X,G=E] 
              - \mathbb E[S\mid \X,G=O,A=1] + \mathbb E[S\mid \X,G=O,A=0] \\
            = & \mathbb E[Y\mid \X,G=O,A=1] - \mathbb E[Y\mid \X,G=O,A=0] \\
              & + \mathbb E[S\mid \X,G=E,A=1] - \mathbb E[S\mid \X,G=E,A=0] 
               - \mathbb E[S\mid \X,G=O,A=1] + \mathbb E[S\mid \X,G=O,A=0]
                \\
            = & \mu_Y^O(1,\X) - \mu_Y^O(0,\X)  
             +  \mu_S^E(1,\X) -  \mu_S^E(0,\X)
             +  \mu_S^O(0,\X) - \mu_S^O(1,\X) ,
        \end{aligned}
    \end{equation}
    where the second equality is based on Assumption \ref{assum: external validity of exp} and the last equality is based on Assumption \ref{assum: internal validity of exp}. 
    This finishes our proofs.
\end{proof}
% \newpage

\section{Proof of Proposition \ref{propo: confounding bias function}}
\label{app: proof propo}

We first restate the proposition as follow:
\begin{proposition}
    Under Assumption \ref{assum: time series equ bias},  $\forall t$ the confounding biases between times $t$ and $t+1$ follow
\begin{equation}
\begin{aligned}
    \omega_{t+1}(\X) = f(\X) \omega_t(\X) ,
\end{aligned}
\end{equation}
where $\omega_t(\X)$ is the confounding bias at time step $t$, defined as $\omega_t(\X) =  \mu_{S_t}^E(1,\X) - \mu_{S_t}^E(0,\X) + 
      \mu_{S_t}^O(0,\X) - \mu_{S_t}^O(1,\X)$.
\end{proposition}

\begin{proof}
    We start from the long-term causal effects at time step $t$ in experimental data $G=E$:
    \begin{equation}
        \begin{aligned}
            & \mathbb E[S_t(1)\mid \X,G=E] - \mathbb E[S_t(0)\mid \X,G=E] \\
            \overset{(a)}{=} & \mathbb E[S_t(1)\mid \X,G=O] - \mathbb E[S_t(0)\mid \X,G=O] \\
            = & \mathbb E[S_t(1)\mid \X,G=O,A=1]p(A=1\mid \X,G=O) 
                + \mathbb E[S_t(1)\mid \X,G=O,A=0]p(A=0\mid \X,G=O) \\
              & - \mathbb E[S_t(0)\mid \X,G=O,A=1]p(A=1\mid \X,G=O) 
                + \mathbb E[S_t(0)\mid \X,G=O,A=0]p(A=0\mid \X,G=O) \\
            = & \mathbb E[S_t(1)\mid \X,G=O,A=1]p(A=1\mid \X,G=O) 
                + \mathbb E[S_t(1)\mid \X,G=O,A=0]p(A=0\mid \X,G=O) \\
              & - \mathbb E[S_t(0)\mid \X,G=O,A=1]p(A=1\mid \X,G=O) 
                + \mathbb E[S_t(0)\mid \X,G=O,A=0]p(A=0\mid \X,G=O) \\
              & + \mathbb E[S_t(1)\mid \X,G=O,A=1]p(A=0\mid \X,G=O)
                - \mathbb E[S_t(1)\mid \X,G=O,A=1]p(A=0\mid \X,G=O) \\
              & + \mathbb E[S_t(0)\mid \X,G=O,A=0]p(A=1\mid \X,G=O)
                - \mathbb E[S_t(0)\mid \X,G=O,A=0]p(A=1\mid \X,G=O) \\
            = & \mathbb E[S_t(1)\mid \X,G=O,A=1] - \mathbb E[S_t(0)\mid \X,G=O,A=0] \\
              & + \{ \mathbb E[S_t(1)\mid \X,G=O,A=0] - \mathbb E[S_t(1)\mid \X,G=O,A=1]\}p(A=0,\X,G=O) \\
              & + \{ \mathbb E[S_t(0)\mid \X,G=O,A=0] - \mathbb E[S_t(0)\mid \X,G=O,A=1]\}p(A=1,\X,G=O) \\
            = & \mathbb E[S_t\mid \X,G=O,A=1] - \mathbb E[S_t\mid \X,G=O,A=0] \\
              & + b_t(1,\X) p(A=0,\X,G=O) - b_t (0,\X) p(A=1,\X,G=O), \\
        \end{aligned}
    \end{equation}
    where the equality $(a)$ is based on Assumption \ref{assum: external validity of exp}. By rewriting the last equality above, we obtain:
     \begin{equation} \label{app: eq omega t}
        \begin{aligned}
            & \mathbb E[S_t(1)\mid \X,G=E] - \mathbb E[S_t(0)\mid \X,G=E] \\
            = & \mathbb E[S_t\mid \X,G=O,A=1] - \mathbb E[S_t\mid \X,G=O,A=0] \\
              & + b_t(1,\X) p(A=0,\X,G=O) - b_t (0,\X) p(A=1,\X,G=O), \\
            \iff 
            &\mathbb E[S_t(1)\mid \X,G=E] - \mathbb E[S_t(0)\mid \X,G=E] - \mathbb E[S_t\mid \X,G=O,A=1] + \mathbb E[S_t\mid \X,G=O,A=0] \\
            = & b_t(1,\X) p(A=0,\X,G=O) - b_t (0,\X) p(A=1,\X,G=O), \\
            \iff &
            \mu_{S_t}^E(1,\X) - \mu_{S_t}^E(0,\X) + 
      \mu_{S_t}^O(0,\X) - \mu_{S_t}^O(1,\X) \\
            = & b_t(1,\X) p(A=0,\X,G=O) - b_t (0,\X) p(A=1,\X,G=O), \\
            \overset{(a)}{\iff} &
            \omega_t(\X)
            = b_t(1,\X) ) p(A=0,\X,G=O) - b_t (0,\X) p(A=1,\X,G=O), \\
        \end{aligned}
    \end{equation}
    where $(a)$ is based on the definition of $\omega_t(\X)$. Then, similarly for time step $t+1$, we have
     \begin{equation}
        \begin{aligned}
           \omega_{t+1}(\X)
             & = b_{t+1}(1,\X) p(A=0,\X,G=O) -b_{t+1}(0,\X) p(A=1,\X,G=O), \\
            & \overset{(a)}{=} f(\X) \left(b_{t}(1,\X) p(A=0,\X,G=O) -b_{t}(0,\X) p(A=1,\X,G=O) \right)  \\
            & \overset{(b)}{=} f(\X) \omega_t(\X),
        \end{aligned}
    \end{equation}
    where $(a)$ is based on Assumption \ref{assum: time series equ bias}, i.e., $ b_{t+1}(a,\X) = f(\X) b_t(a,\X) $, and $(b)$ is based on Eq. \eqref{app: eq omega t} and finishes our proof.
\end{proof}

% \newpage

\section{Proof of Theorem \ref{theo: time series identifi}}
\label{app: proof of time series iden}
% We provide the proof Theorem \ref{theo: identifi}.
We first restate Theorem \ref{theo: time series identifi} as follows:

\begin{theorem}
Suppose Assumptions \ref{assum: consist}, \ref{assum: positi}, \ref{assum: internal validity of obs}, \ref{assum: internal validity of exp}, \ref{assum: external validity of exp} and \ref{assum: time series equ bias} hold, then $\tau(\x)$ can be identified as follows:
    \begin{equation}
      \begin{aligned}
       & \tau(\x) \\
            = & \mathbb E[Y(1)-Y(0)\mid \X=\x] \\
            = & \mu_Y^O(1,\x) - \mu_Y^O(0,\x)
            +  f^\mu(\X)  \left(\mu_{S_T}^E(1,\x) -  \mu_{S_T}^E(0,\x)
             +  \mu_{S_T}^O(0,\x) - \mu_{S_T}^O(1,\x) \right) .
      \end{aligned}
    \end{equation}
\end{theorem}

\begin{proof}
    \begin{equation}
        \begin{aligned}
            \tau(\X) 
            = & \mathbb E[Y(1)\mid \X,G=O] - \mathbb E[Y(0)\mid \X,G=O] \\
            = & \mathbb E[Y(1)\mid \X,G=O,A=1]p(A=1\mid \X,G=O) 
                + \mathbb E[Y(1)\mid \X,G=O,A=0]p(A=0\mid \X,G=O) \\
              & - \mathbb E[Y(0)\mid \X,G=O,A=1]p(A=1\mid \X,G=O) 
                + \mathbb E[Y(0)\mid \X,G=O,A=0]p(A=0\mid \X,G=O) \\
            = & \mathbb E[Y(1)\mid \X,G=O,A=1]p(A=1\mid \X,G=O) 
                + \mathbb E[Y(1)\mid \X,G=O,A=0]p(A=0\mid \X,G=O) \\
              & - \mathbb E[Y(0)\mid \X,G=O,A=1]p(A=1\mid \X,G=O) 
                + \mathbb E[Y(0)\mid \X,G=O,A=0]p(A=0\mid \X,G=O) \\
              & + \mathbb E[Y(1)\mid \X,G=O,A=1]p(A=0\mid \X,G=O)
                - \mathbb E[Y(1)\mid \X,G=O,A=1]p(A=0\mid \X,G=O) \\
              & + \mathbb E[Y(0)\mid \X,G=O,A=0]p(A=1\mid \X,G=O)
                - \mathbb E[Y(0)\mid \X,G=O,A=0]p(A=1\mid \X,G=O) \\
            = & \mathbb E[Y(1)\mid \X,G=O,A=1] - \mathbb E[Y(0)\mid \X,G=O,A=0] \\
              & + \{ \mathbb E[Y(1)\mid \X,G=O,A=0] - \mathbb E[Y(1)\mid \X,G=O,A=1]\}p(A=0,\X,G=O) \\
              & + \{ \mathbb E[Y(0)\mid \X,G=O,A=0] - \mathbb E[Y(0)\mid \X,G=O,A=1]\}p(A=1,\X,G=O) \\
            = & \mathbb E[Y\mid \X,G=O,A=1] - \mathbb E[Y\mid \X,G=O,A=0] \\
              & + b_{T+\mu}(1,\X) p(A=0,\X,G=O) - b_{T+\mu}(0,\X) p(A=1,\X,G=O) \\
        \end{aligned}
    \end{equation}
    where the first equality is based on Assumption \ref{assum: external validity of exp}, i.e., $G \Vbar Y(a) \mid  \X$, and last equality is based on Assumption \ref{assum: time series equ bias}.
    
    Further based on Assumption \ref{assum: time series equ bias}, we have $b_{T+\mu}(a,\X)=f^\mu(\X) b_{T}(a,\X)$. Then, we rewrite the equality above as:
    \begin{equation} \label{eq: time tau decom}
        \begin{aligned}
             \tau(\X) 
            = & \mathbb E[Y(1)\mid \X,G=O] - \mathbb E[Y(0)\mid \X,G=O] \\
            = & \mathbb E[Y\mid \X,G=O,A=1] - \mathbb E[Y\mid \X,G=O,A=0] \\
              & + f^\mu(\X) b_{T}(1,\X) p(A=0,\X,G=O) - f^\mu(\X) b_{T}(0,\X) p(A=1,\X,G=O) \\
        \end{aligned}
    \end{equation}
    Similarly, for short-term ITE at last time step $T$ we have:
    \begin{equation} \label{eq: time stce decom}
        \begin{aligned}
            & \mathbb E[S_T(1)\mid \X] - \mathbb E[S_T(0)\mid \X] \\
            = & \mathbb E[S_T\mid \X,G=O,A=1] - \mathbb E[S_T\mid \X,G=O,A=0] \\
              & + \{ \mathbb E[S_T(1)\mid \X,G=O,A=0] - \mathbb E[S_T(1)\mid \X,G=O,A=1]\}p(A=0,\X,G=O) \\
              & + \{ \mathbb E[S_T(0)\mid \X,G=O,A=0] - \mathbb E[S_T(0)\mid \X,G=O,A=1]\}p(A=1,\X,G=O) \\
            = & \mathbb E[S_T\mid \X,G=O,A=1] - \mathbb E[S_T\mid \X,G=O,A=0] \\
              & + b_{T}(1,\X) p(A=0,\X,G=O) - b_{T}(0,\X) p(A=1,\X,G=O)
        \end{aligned}
    \end{equation}
    
    Then, combining Eq. \ref{eq: time tau decom} and \ref{eq: time stce decom}, we have
    \begin{equation} 
        \begin{aligned}
            & \tau(\X) \\ 
            = & \mathbb E[Y\mid \X,G=O,A=1] - \mathbb E[Y\mid \X,G=O,A=0] + 
             f^\mu(\X)   \mathbb E[S_T(1)\mid \X] -  f^\mu(\X) \mathbb E[S_T(0)\mid \X] \\
            & -  f^\mu(\X) \mathbb E[S_T\mid \X,G=O,A=1] -  f^\mu(\X)  \mathbb E[S_T\mid \X,G=O,A=0] \\
            = & \mathbb E[Y\mid \X,G=O,A=1] - \mathbb E[Y\mid \X,G=O,A=0] + 
             f^\mu(\X)   \mathbb E[S_T\mid \X,G=E,A=1]
             \\ & -  f^\mu(\X) \mathbb E[S_T\mid \X,G=E,A=0] 
             -  f^\mu(\X) \mathbb E[S_T\mid \X,G=O,A=1] -  f^\mu(\X)  \mathbb E[S_T\mid \X,G=O,A=0] \\
            = & \mu_Y^O(1,\x) - \mu_Y^O(0,\x) 
            +  f^\mu(\X) \left(\mu_{S_T}^E(1,\x) -  \mu_{S_T}^E(0,\x)
             +  \mu_{S_T}^O(0,\x) - \mu_{S_T}^O(1,\x) \right)
        \end{aligned}
    \end{equation}
    where the second equality is based on Assumption \ref{assum: external validity of exp}. This finishes our proof.
\end{proof}

\section{Proof of Lemma \ref{lemma: convergence of f}}
\label{app: proof f conver}

We first state the following lemma that is used in the proof of Lemma \ref{lemma: convergence of f}.

\begin{lemma}  \label{app: lemma conv of omega}
suppose Assumptions \ref{assum: consist}, \ref{assum: positi}, \ref{assum: internal validity of obs}, \ref{assum: internal validity of exp}, \ref{assum: external validity of exp}, \ref{assum: time series equ bias} hold, then $\forall t$, 
        \begin{equation} 
      \begin{aligned}
       & \hat \omega_t (\x) -\omega_t (\x) 
            = O_p \left( 
            r_{\mu _{S_t}^E}(n) + r_{\mu _{S_t}^O}(n)
            \right).
      \end{aligned}
    \end{equation}
    where $r_{\circ}(n)$ denotes the risk of nuisance function $\circ$, e.g., $r_{\mu _{S_t}^E}(n)$ correspondingly to $\mu _{S_t}^E$, and further under Assumption \ref{asmp: smooth}, we have
    \begin{equation} 
      \begin{aligned}
       & \hat \omega_t (\x) -\omega_t (\x) 
            = O_p \left( 
            n^{\frac{-\alpha}{2\alpha+d}} + n^{\frac{-\beta}{2\beta+d}}
            \right).
      \end{aligned}
    \end{equation}
\end{lemma}

\begin{proof}
    The lemma is immediately proved by the form of $\omega_t(\x)$ as 
    $\omega_t(\x) = \mu _{S_t}^E(1,\x) - \mu _{S_t}^E(0,\x) + \mu _{S_t}^O(0,\x) - \mu _{S_t}^O(1,\x)$.
\end{proof}

We now restate our Lemma \ref{lemma: convergence of f} as follows.

\begin{lemma} 
    Suppose the training steps S1 and S2 are train on two independent datasets of size $n$ respectively, and suppose Assumptions \ref{assum: consist}, \ref{assum: positi}, \ref{assum: internal validity of obs}, \ref{assum: internal validity of exp}, \ref{assum: external validity of exp}, \ref{assum: time series equ bias}, \ref{asmp: smooth}, and \ref{asmp: bound} hold, then we have
        \begin{equation} \label{eq: f rate}  
      \begin{aligned}
       \hat f(\x) - & f(\x) 
         =  O_p \left( 
            (\frac{1}{(T-1)n})^{\frac{\eta}{2\eta+d}} 
            % \right.\\  & \left. 
            +  (\frac{1}{(T-1)n})^{\frac{\alpha}{2\alpha+d}}
            +  (\frac{1}{(T-1)n})^{\frac{\beta}{2\beta+d}}
            \right),
      \end{aligned}
    \end{equation}
    which attains the oracle rate if 
     $\min \{\alpha, \beta \} \geq \eta$.
\end{lemma}

\begin{proof}
     We apply Proposition 1 in Kennedy et~al. \cite{kennedy2023towards}, yielding that
    \begin{equation} \label{app: eq f decomp}
    \begin{aligned}
       \hat f(\x) - f(\x) 
       = & (\hat f(\x) - \tilde f(\x)) + (\tilde f(\x) - f(\x))\\
       = &  (\hat f(\x) - \tilde f(\x)) +  O_p(R_n^*(\x)) \\
       = &  \hat{\mathbb E}_n[\hat r(\X)\mid \X=\x] + o_p(R_n^*(\x))  +  O_p(R_n^*(\x)) \\
    \end{aligned}
    \end{equation}
    where $\hat r (\x) = \mathbb E[\hat f(\X) \mid \X=\x ] - f(\x)$, and $\mathcal R_n^*(x)$ is the oracle risk of second-stage regression and further under Assumption \ref{asmp: smooth}, we know $f$ is $\eta$-smooth, thus $\mathcal R_n^*(x)=O_p( (\frac{1}{(T-1)n})^{\frac{\eta}{2\eta+d}})$. 
    Here, $f$ is a time-series model ($1$-order autoregressive), optimized with $T-1$-length steps and $n$ samples, as shown in Eq. \eqref{eq: s2}. 
    According to Proposition \ref{propo: confounding bias function}, $\forall t$, $f(\x)=\frac{\omega_{t+1} (\x)}{\omega _t (\x)}$. We first prove the rate of $\frac{\hat \omega_{t+1} (\x)}{\hat \omega _t (\x)}$ for a fixed $t$ as follows.
    \begin{equation} 
        \begin{aligned}
            & (\frac{\hat \omega_{t+1} (\x)}{\hat \omega _t (\x)} - \frac{ \omega_{t+1} (\x)}{ \omega _t (\x)} )^2 \\
            = & \left( \frac{ \hat{ \omega}_{t+1}(\x) \omega_t(\x) -  \omega_{t+1}(\x) \hat{ \omega}_t(\x)}{ \omega_t(\x) \hat{ \omega}_t(\x)} \right)^2 \\
            = & \left(  \frac{ \hat{ \omega}_{t+1}(\x) \omega_t(\x) -  \omega_{t+1}(\x) \hat{ \omega}_t(\x)
            + \omega_{t+1}(\x) \omega_t(\x)
            - \omega_{t+1}(\x) \omega_t(\x)
            }{ \omega_t(\x) \hat{ \omega}_t(\x)}  \right)^2 \\
            = & \left(  \frac{ (\hat{ \omega}_{t+1}(\x)- \omega_{t+1}(\x)) \omega_t(\x) 
            + \omega_{t+1}(\x) ( \omega_t(\x) -\hat{ \omega}_t(\x))
            }{ \omega_t(\x) \hat{ \omega}_t(\x)}  \right)^2\\
            \overset{(a)}{\leq} & \frac{2 (\hat{ \omega}_{t+1}(\x)- \omega_{t+1}(\x))^2 \omega_t^2(\x) 
            + \omega_{t+1}^2(\x) ( \omega_t(\x) -\hat{ \omega}_t(\x))^2
            }{ \omega_t^2(\x) \hat{ \omega}_t^2(\x)} \\
            = & \frac{ 2 }{  \hat{ \omega}_t^2(\x)} (\hat{ \omega}_{t+1}(\x)- \omega_{t+1}(\x))^2
            + \frac{ 2 \omega_{t+1}^2(\x) }{ \omega_t^2(\x) \hat{ \omega}_t^2(\x)} ( \omega_t(\x) -\hat{ \omega}_t(\x))^2  \\
            \overset{(b)}{\asymp}  & (\hat{  \omega}_{t+1}(\x)-  \omega_{t+1}(\x))^2 + ( \omega_{t}(\x) -\hat{  \omega}_{t}(\x))^2;
    \end{aligned}
    \end{equation}
    where the inequality $(a)$ is based on $(a+b)^2 \leq 2 (a^2 + b^2)$, and $(b)$ is based on Assumption \ref{asmp: bound}. Under Assumption \ref{asmp: bound}, $\mu_{S_t}^O$, $\mu_{S_t}^E$ and their estimates are all bounded, and thus $  \omega_{t}(\x)$ and its estimates are also bounded, thus $(b)$ holds. Then, applying Lemma \ref{app: lemma conv of omega}, we know for a fixed time step $t$, $\frac{\hat \omega_{t+1} (\x)}{\hat \omega _t (\x)} - \frac{ \omega_{t+1} (\x)}{ \omega _t (\x)} = O_p( n^{\frac{-\alpha}{2\alpha+d}} + n^{\frac{-\beta}{2\beta+d}})$. Then, for $\hat r (\x)$, we have the effective sample is of size $(T-1)n$, and then combining with Eq. \eqref{app: eq f decomp} and $\mathcal R_n^*(x)=O_p( (\frac{1}{(T-1)n})^{\frac{\eta}{2\eta+d}})$, we obtain the desired result. And attaining the oracle rate requires $(\frac{1}{(T-1)n})^{\frac{\eta}{2\eta+d}} \geq \max \{ (\frac{1}{(T-1)n})^{\frac{\alpha}{2\alpha+d}}, (\frac{1}{(T-1)n})^{\frac{\beta}{2\beta+d}}\}$, yielding $\min \{\alpha, \beta\} \geq \eta$.
\end{proof}    

\section{Proof of Theorem \ref{theo: convergence}}
\label{app: proof tau conver}

We first restate  Theorem \ref{theo: convergence} as follows:

\begin{theorem} 
Suppose Lemma \ref{lemma: convergence of f} hold, then we have
    \begin{equation} 
      \begin{aligned}
      & \hat \tau(\x) -   \tau(\x) \\
         = &   O_p \left(  
          n^{-\frac{\gamma}{2\gamma+d}} +
          n^{-\frac{\alpha}{2\alpha+d}} +
          n^{-\frac{\beta}{2\beta+d}} +
           \frac{\mu}{ ((T-1)n)^{\frac{\eta}{2\eta+d}} }  
           + \frac{\mu}{ ((T-1)n)^{\frac{\alpha}{2\alpha+d}} }  + 
           \frac{\mu}{ ((T-1)n)^{\frac{\beta}{2\beta+d}} }  
            \right).
      \end{aligned}
    \end{equation}
\end{theorem}

\begin{proof}
    As stated in Lemma \ref{lemma: convergence of f}, we have 
    \begin{equation} 
      \begin{aligned}
       & \hat f(\x) - f(\x) 
            = O_p \left( 
            (\frac{1}{(T-1)n})^{\frac{\eta}{2\eta+d}} +
            (\frac{1}{(T-1)n})^{\frac{\alpha}{2\alpha+d}} +
            (\frac{1}{(T-1)n})^{\frac{\beta}{2\beta+d}}
            \right),
      \end{aligned}
    \end{equation}
    and thus 
    \begin{equation} \label{app: eq1 proof of theo conve}
      \begin{aligned}
       & \hat f^\mu(\x) - f^\mu(\x) 
            = O_p \left( 
            (\frac{\mu}{(T-1)n})^{\frac{\eta}{2\eta+d}} +
            (\frac{\mu}{(T-1)n})^{\frac{\alpha}{2\alpha+d}} +
            (\frac{\mu}{(T-1)n})^{\frac{\beta}{2\beta+d}}
            \right).
      \end{aligned}
    \end{equation}
    Under Assumption \ref{asmp: smooth}, we have 
    \begin{equation} \label{app: eq2 proof of theo conve}
      \begin{aligned}
       \hat \mu_Y^O(a,\x) - \mu_Y^O(a,\x)=O_p(n^{-\frac{\gamma}{2\gamma+d}})
      \end{aligned}
    \end{equation}
    and
    \begin{equation} 
      \begin{aligned}
       \hat \mu_{S_T}^E(a,\x) - \mu_{S_T}^E(a,\x)=O_p(n^{-\frac{\alpha}{2\alpha+d}})
      \end{aligned}
    \end{equation}
    and 
    \begin{equation} 
      \begin{aligned}
       \hat \mu_{S_T}^O(a,\x) - \mu_{S_T}^O(a,\x)=O_p(n^{-\frac{\beta}{2\beta+d}}).
      \end{aligned}
    \end{equation}
    Then, let $a:=f^\mu(\X) $,$\hat a = \hat f^\mu(\X) $, $b=\mu_{S_T}^E(1,\x) -  \mu_{S_T}^E(0,\x)
             +  \mu_{S_T}^O(0,\x) - \mu_{S_T}^O(1,\x)$, and similarly for $\hat b$, and under Assumption \ref{asmp: bound}, $a$, $b$, $\hat a$ and $\hat b$ are all bounded.
    We analyze the term
    \begin{equation} 
      \begin{aligned}
      ( a b - \hat a \hat b )^2
       = & (ab  - a \hat b + a \hat b - \hat a \hat b )^2\\
       = & (a (b-\hat b) + (a-\hat a )\hat b)^2 \\
       \overset{(a)}{\leq} & 2 a^2 (b-\hat b )^2 + 2 \hat b^2 (a - \hat a)^2  \\
       \overset{(a)}{\asymp} & (b-\hat b)^2 + (a - \hat a)^2,
      \end{aligned}
    \end{equation}
    where the inequality $(a)$ is based on $(a+b)^2 \leq 2(a^2+b^2)$ and $(b)$ is based on the boundedness assumption. Then we have 
    \begin{equation} 
      \begin{aligned}
      ab-\hat a \hat b =  O_p \left( 
            (\frac{\mu}{(T-1)n})^{\frac{\eta}{2\eta+d}} +
            (\frac{\mu}{(T-1)n})^{\frac{\alpha}{2\alpha+d}} +
            (\frac{\mu}{(T-1)n})^{\frac{\beta}{2\beta+d}} + 
            n^{-\frac{\alpha}{2\alpha+d}} +
            n^{-\frac{\beta}{2\beta+d}}
            \right).
      \end{aligned}
    \end{equation}
    Hence, the result is immediately proved by the form of $\tau(x)$ as  $\tau(x) = \mu_Y^O(1,\x) - \mu_Y^O(0,\x) +  f^\mu(\X) \left(\mu_{S_T}^E(1,\x) -  \mu_{S_T}^E(0,\x)
             +  \mu_{S_T}^O(0,\x) - \mu_{S_T}^O(1,\x) \right) $ and the result of Lemma \ref{lemma: convergence of f}.
\end{proof}

\end{document}